\documentclass{article}

\usepackage[utf8]{inputenc} 
\usepackage[T1]{fontenc}    
\usepackage{hyperref}       
\usepackage{url}            
\usepackage{booktabs}       
\usepackage{amsfonts}       
\usepackage{nicefrac}       
\usepackage{microtype}      
\usepackage[margin=1.5in]{geometry}
\usepackage{natbib}

\usepackage[bottom]{footmisc}
\usepackage[mathscr]{euscript}
\usepackage{subfig}
\usepackage{graphicx}
\usepackage{mathptmx}
\usepackage{amsmath}
\usepackage{amssymb}
\usepackage{amsthm}
\usepackage{parskip}
\usepackage{float}
\usepackage{xcolor}
\newcommand{\comment}[1]{}

\newtheorem{thm}{Theorem}[section]
\newtheorem{lem}[thm]{Lemma}

\theoremstyle{definition}
\newtheorem{defn}{Definition}[section]

\title{\vspace{-2cm}Recovering from Biased Data: Can Fairness Constraints Improve Accuracy?}

\author{ Avrim Blum\thanks{Supported in part by the National Science Foundation under grants CCF-1815011 and CCF-1733556.}\ ,\ \  Kevin Stangl\thanks{Supported in part by the National Science Foundation under grant CCF-1815011.} \\
Toyota Technological Institute at Chicago\\
 Chicago, IL 60637 \\
  \texttt{\{avrim,kevin\}@ttic.edu} \\
}

\begin{document}

\maketitle
\begin{abstract}
\normalsize
Multiple fairness constraints have been proposed in the literature,
 motivated by a range of concerns about how demographic groups
might be treated unfairly by machine learning classifiers. 
In this work we consider a different motivation; learning from biased training data. 
We posit several ways in which training data may be biased, including having a more noisy or negatively biased
labeling process on members of a disadvantaged group, or a 
decreased prevalence of positive or negative examples from the disadvantaged group, or both. 

Given such biased training data, Empirical Risk Minimization (ERM) 
 may produce a classifier that not only is biased but also has suboptimal 
accuracy on the true data distribution.
We examine the ability of fairness-constrained ERM to correct this problem.
In particular, we find that the Equal Opportunity fairness constraint (Hardt, Price, and Srebro 2016)
 combined with ERM will provably recover the Bayes Optimal Classifier under a range of bias models. 
We also consider other recovery methods including re-weighting the training data, Equalized Odds, and Demographic Parity, 
and Calibration.
These theoretical results provide additional motivation for considering fairness interventions even if an actor cares primarily about accuracy.
\end{abstract}

\section{Introduction}
Machine learning (typically supervised learning) systems are automating decisions that affect individuals
in sensitive and high stakes domains such as credit scoring \citep{scoredsociety} and  bail assignment \citep{machinebias,flores}. 
This trend toward greater automation of decisions
has produced concerns that learned models may reflect and  
amplify existing social bias or disparities in the training data. 
Examples of possible bias in learning systems include the Pro-Publica investigation of COMPAS (an actuarial risk instrument) \citep{machinebias}, accuracy disparities in computer vision systems \citep{shades}, and gender bias in word vectors \citep{kalai}. 

In order to address observed disparities in learning systems, an approach that has developed into a significant body of work is to add demographic constraints 
to the learning problem that encode criteria that a fair classifier ought to satisfy. 

Multiple constraints have been proposed in the literature \citep{eodds, dwork12}, each encoding a different type of unfairness one might be concerned about, and there has been substantial work on understanding their relationships to each other, including incompatibilities between the fairness requirements \citep{costfairness, chouldechova, inherent,pleiss2017fairness}. 

In this work, we take a different angle on the question of fairness.  Rather than argue whether or not these demographic constraints encode intrinsically desirable properties of a classifier, 
we instead consider their ability to help a learning algorithm to recover from biased training data and to produce a {\em more accurate} classifier. 

In particular, adding a constraint (such as a fairness constraint) to an optimization problem (such as ERM) would typically 
result in a lower quality solution. However, if the objective being optimized is skewed (e.g., because training data is corrupted or not drawn from the correct distribution) then such constraints might actually help prevent the optimizer from being led astray, with a higher quality solution when accuracy is \textit{measured on the true distribution}. 

More specifically, we consider a binary classification setting in which data points correspond to individuals, some of whom are members of an advantaged Group A and the rest of the individuals are members of a disadvantaged Group B. 
We want to make a decision such as deciding whether to offer a candidate a loan or admission to college.  
We have access to labeled training data consisting of $(x,y)$ pairs where $x$ is some set of features corresponding to an individual and $y$ is a label we want to predict for new individuals. 

The concern is that the training data is potentially biased against Group $B$ 
in that \emph{the training data systematically misrepresents
the true distribution over features and labels in Group $B$}, while the training data for Group $A$ is drawn
from the true distribution for Group $A$. 

We consider several natural ways this might occur.  One way is that members of the disadvantaged group might show up in the training data at a lower rate than their true prevalence in the population, and worse, {\em this rate might depend on their true label}. 

For instance, if the positive examples of Group B appear at a much lower rate in the training data than the negative examples of Group B (which might occur for cultural reasons or due to other options available to them),
then ERM might learn a rule that classifies all or most members of Group B as negative.  

A second form of bias in the training data we consider is bias in the labeling process.  
Human labelers might have inherent biases causing some positive members of Group B in the training data to be mislabeled as negative, which again could cause unconstrained ERM to be more pessimistic than it should be.  Alternatively, both processes might occur together. 

We examine the ability of fairness constraints to help an ERM learning method recover from these problems.

\subsection{Summary of Results}
Our main result is that ERM subject to the \textbf{Equal Opportunity} fairness constraint \citep{eodds} recovers the true Bayes Optimal hypothesis under a wide range of bias models, making it an attractive choice even for decision makers whose overall concern is purely about accuracy on the true data distribution. 


In particular, we assume that under the true data distribution, the Bayes-Optimal classifiers $h_A^*$ and $h_B^*$ classify the same fraction $p$ of their respective populations as positive\footnote{$p=P_{\mathscr{D}_A} (h_{A}^{*}(x) = 1) = P_{\mathscr{D}_B} (h_{B}^{*} (x) = 1)$. 
We will allow the classifiers to make decisions based on group membership or alternatively assume we have sufficiently rich data to implicitly infer the group attribute.},
$h_A^*$ and $h_B^*$ have the same error rate $\eta$ on their respective populations, 
and that these errors are uniformly distributed.  

However, during the training process we do not have access to the true distribution. We only have access to a biased distribution in a way that implicates the distinct social groups and causes the classifier to be overly pessimistic on individuals from Group $B$.

We prove that, subject to the above conditions on $h_{A}^{*}$ and $h_{B}^{*}$, even with substantially corrupted training data either due to the under-representation of positive examples in Group B or a substantial fraction of positive examples in Group B mislabeled as negative, or both, the Equality of Opportunity 
fairness constraint will enable ERM to learn the Bayes Optimal classifier $h^{*}=(h_{A}^{*},h_{B}^{*})$, subject to a pair of inequalities ensuring that the labels are not too noisy and Group $A$ has large mass.

Expressed another way, this means that 
\emph{the lowest error classifier on the biased data satisfying Equality of Opportunity is the Bayes Optimal Classifier on the un-corrupted data.} These results provide additional motivation for considering fairness interventions, and in particular Equality of Opportunity, even if one cares primarily about accuracy.

 Other related fairness notions such as Equalized Odds, Demographic Parity, and Calibration 
 do not succeed in recovering the Bayes Optimal classifier under such broad conditions.
 In fact, we show that given data subject to Under-Representation Bias, Calibration can actually {\em amplify} the effects of the bias, and so can be worse than doing nothing  and instead learning with plain ERM (see Section \ref{sec:underrep}).

Our results are in the infinite sample limit and we suppress issues of sample complexity\footnote{Our notion of sample complexity is typical. 
Let $S$ be the biased training data-set and $ERM_{\mathscr{H}}(S) = \hat{h}$.
Given $\epsilon, \delta > 0$, $m(\epsilon, \delta)$ samples ensures with
probability greater than $1-\delta$ that
$L_{\mathscr{D}}(\hat{h}) \leq L_{\mathscr{D}}(h^*) + \epsilon$.
}
in order to focus on the core phenomenon of the data source being unreliable. 

\subsection{Related Work}
This paper is directly motivated by a model of implicit bias in ranking \citep{implicit}.
In that paper, the training data for a hiring process is systematically corrupted against minority candidates
and a method to correct this bias 
increases both the quality of the accepted candidate and the fraction of hired minority candidates.
However, that fairness intervention, the Rooney Rule, does not immediately translate to a general learning setting. 

Our results avoid triggering the known impossibility results between high accuracy and satisfying fairness criteria 
\citep{chouldechova, inherent} by assuming we have equal base rates across groups. 
This assumption may not be realistic in all settings, however there are settings where bias concerns arise and there is empirical evidence that base rates are equivalent across the relevant demographic groups, e.g. highly differential arrest rates for some alleged crimes that have similar occurrence rates across groups \citep{predictandserve, dirtydata}. 

Within the fairness literature there are several approaches similar to ours. 
In particular, our concern with positive examples not appearing in the training data 
is similar in effect to a selective labels problem \citep{kleinbergselective}. 
\citep{chouldechovaselectivelabels} uses data augmentation to experimentally improve generalization under selective label bias. 

\citep{impossibility, discriminative} also consider the training and test data distribution gap we experience in our model and posit differing interpretations of fairness constraints under different worldviews. 
While we do not explicitly use the terminology in these papers, we believe our view of the gap between the true distribution and the training time distribution is aligned with Friedler et al's concept of the gap between the construct space and the observed space. 

Our second bias model, Labeling Bias, is similar to \citep{labelbias}.
In that paper, the bias phenomenon is that a biased labeler makes poor decisions on the disadvantaged group and intervenes with a reweighting technique, one that is more complex than our Re-Weighting intervention.
However, that paper does not consider the interaction of biased labels with different groups appearing in the data at different rates as a function of their labels.


\section{Model}
In this section we describe our learning model, how bias enters the data-set, and the fairness interventions we consider. 

We assume the data lies in some instance space
$\mathscr{X}$, such as $\mathscr{X} = \mathbb{R}^d$.
There are two demographic groups in the population, Group $A$ and Group $B$.
Their proportions in the population are given by
 $P(x \in A) = 1-r$ and $P(x \in B)= r$ for $r \in (0,1)$.
$x \in A$ can be read as individual $x$ in demographic Group $A$.
Group $B$ is the disadvantaged group that suffers from the effects of the bias model.

Assume there is a special coordinate of the feature vector $x$ that denotes group membership.
The  data distribution is given by 
$\mathscr{D}$, and is a pair distributions $(\mathscr{D}_{A}, \mathscr{D}_B)$,
with $\mathscr{D}_{A}$ determining how $x \in A$ is distributed and $\mathscr{D}_B$ determining how $x \in B$ is distributed.

\subsection{True Label Generation: } \label{truelabels}
 Now we describe how the true labels for individuals are generated. 
Assume there exists a pair of Bayes Optimal Classifiers $h^* = (h_{A}^*, h_{B}^*)$ with $h_{A}^{*}, h_{B}^{*} \in \mathscr{H}: \mathscr{X} \rightarrow \{0,1\}$.

We assume that the  Bayes Optimal classifier $h_B^*$ for Group B may be different from the Bayes Optimal classifier $h_A^*$ for Group A.  
If $h_A^*$ was also optimal for Group B,
then  we can just learn $h^*$ for both Groups $A$ and $B$ using data only from Group $A$ and biased data concerns fade away.
Thus we are learning a pair of classifiers, one for each demographic group. 

When generating samples, first we draw a data-point $x$. 
With probability $1-r$,
 $x \sim \mathscr{D}_{A}$ (and thus $x \in A$)
and with probability $r$, $x \sim \mathscr{D}_{B}$ (so $x \in B$).

Once we have drawn a data-point $x$, we model the true labels as being produced as follows; 
evaluate $h^{*}(x)$, using the classifier corresponding to the demographic group of $x$.
If $x \in A$, then $h^* (x) = h^{*}_{A}(x)$. If $x \in B$, then $h^*(x) = h^{*}_{B}(x)$.
However, we assume that $h^*$ is not perfect and independently with probability $\eta$, the true label of $x$ does
not correspond to the prediction $h^{*}(x)$.
\[y = y(x) = 
\begin{cases}
& \neg \; h^{*}(x) \quad \text{with probability} \quad \eta \\
& h^*(x) \quad \text{ w.p. } \quad 1-\eta \\
\end{cases}
\]
The labels $y$ after this flipping process
are the \textit{true labels} of the training data.\footnote{Note this label model is equivalent to
the Random Classification Noise model \citep{angluin}.
However the key interpretative difference is that in RCN, $h^*(x)$ is the correct label and those that get flipped are noise, but in our case 
the $y$ are the
true labels and $h^*$ is merely the Bayes-Optimal classifier given the observable features.
}
 We assume that $p=P(h^*_{A}(x) = 1 | x \in A) = P(h_{B}^{*}(x) = 1 | x \in B)$.
This combined with the assumption that $\eta$ is the same for classifiers from both groups implies that the two groups
have equal base rates (fraction of positive samples) i.e $p(1-\eta)+(1-p)\eta$.
 
 We denote this label model as $(x,y) \sim P_{\mathscr{D},r}(h^*, \eta)$ for a pair of classifiers
$h^*= (h_{A}^*, h_{B}^*)$ with $h_{A}^*, h_{B}^* \in \mathscr{H}$
where $\mathscr{H}: \mathscr{X} \rightarrow \{ 0,1 \}$ is some hypothesis class with finite VC dimension.

\subsection{Biased Training Data}
Now we consider how bias enters the data-set. 
Consider the example of hiring where the main failure mode will be a classifier that is too negative on the disadvantaged group.
We explore several different bias models to capture potential ways the data-set could become biased. 

The first bias model we call \textbf{Under-Representation Bias}. 
In this model, the positive examples from Group $B$ are under-represented in the training data. 

Specifically, the biased training data is drawn as follows:
\begin{enumerate}
\item $m$ examples are sampled from the distribution $\mathscr{D}$. 
Thus each $x \sim \mathscr{D}$. 
\item The label $y$ for each $x$ is generated according to the 
label process from Section \ref{truelabels} with hypothesis $h^*=(h_{A}^*, h_{B}^*)$ and $\eta$. 
\item For each pair $(x,y)$, if $x \in B$ and $y=1$, then the data-point $(x,y)$ is discarded from our training set independently with probability $1-\beta$. 
\end{enumerate}
Thus we see fewer positive examples from Group $B$ in our training data.
$\beta$ is the probability a positive example from Group $B$ stays in the training data and $  1 > \beta> 0 $.

If $\eta=0$, then the positive and negative regions of $h^{*}$ are strictly disjoint, so if we draw sufficiently many examples,
with high probability, we will see enough positive examples in the positive domain of $h^{*}$ 
to find a low empirical error classifier that is equivalent to $h^*$.\footnote{We would learn with ERM and Uniform Convergence, using the fact that 
$\mathscr{H}$ has finite VC-dimension.}

In contrast for non-zero $\eta$, our label model interacting with the bias model can induce a problematic phenomenon that fools the ERM classifier.
For non-zero $\eta$ there is error even for the Bayes Optimal Classifier $h^*$ and thus
in the positive region of the classifier there are positive examples mixed with negative examples. 
The fraction of negative examples is amplified by the bias process. 

If $\beta$ is sufficiently small, there could in fact be more negative examples of Group B  than positive examples in the positive region of $h_B^*$.  If this occurs, then the bias model will snap the unconstrained ERM optimal hypothesis (optimal on the biased data) to classifying all individuals from Group $B$ as negatives. 
This can be observed in Figure \ref{underrep}.
\begin{figure}
  \centering
  \subfloat[Un-Corrupted Data]{{\includegraphics[width=5cm]{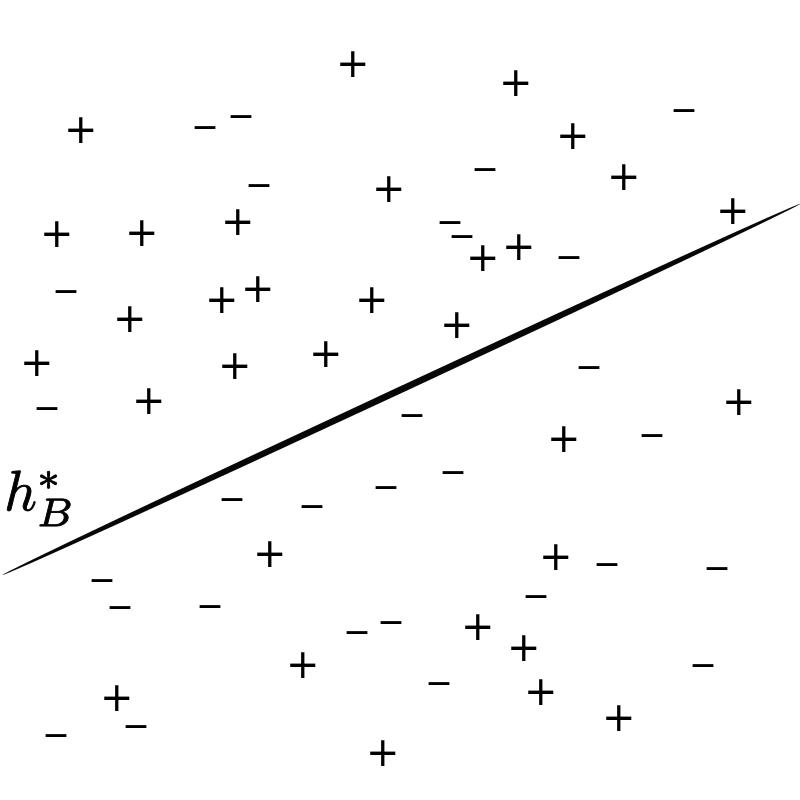}}}
\qquad
\subfloat[Corrupted Data: Under-Representation Bias]{{\includegraphics[width=5cm]{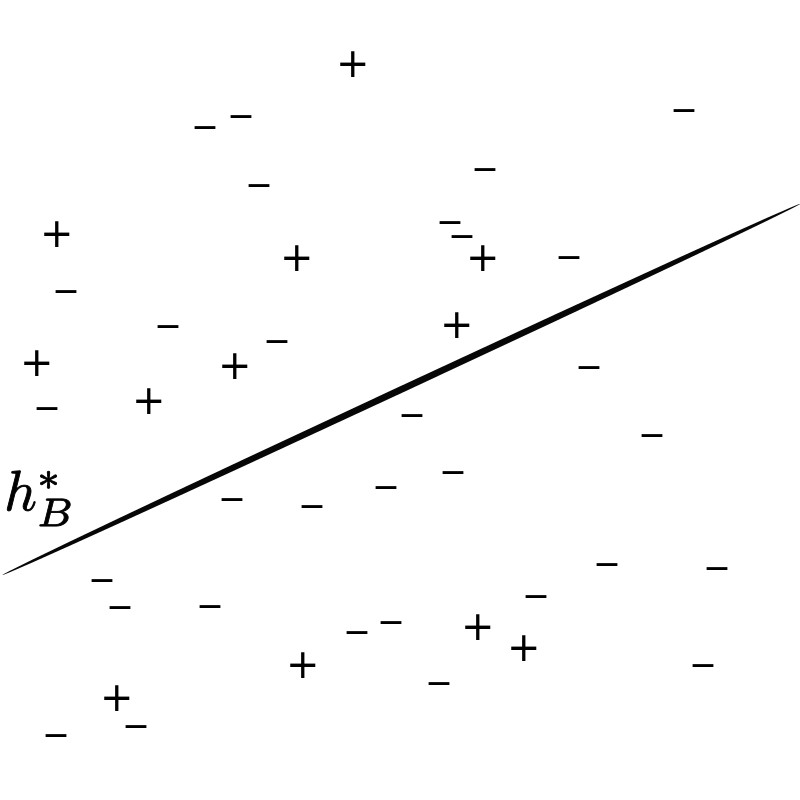}}}
\caption{The schematic on the left displays data points with $p=1/2$, $h^{*}_{B}$ as a hyperplane, and $\eta=1/3$. 
The schematic on the right displays  data drawn from the same distribution subject to the Under-Representation Bias with $\beta_{POS}=1/3$.
Now there are more negative examples than positive examples above the hyperplane so the lowest error hypothesis classifies all examples on the right as negative.}
\label{underrep}
\end{figure}
             
Under-Representation Bias is related to selective labels in \citep{kleinbergselective} since we are learning on a filtered distribution where
the filtering process is correlated with the group label.
Our model is functionally equivalent to over-representing the negatives of the  in the training data, an empirical phenomenon observed in \citep{dirtydata}. 

 \subsection{Alternative Bias Model: Labeling Bias}
We now consider a bias model that captures the notion of implicit bias, which we call \textbf{Labeling Bias.}  In particular, a possible source of bias in machine learning is the label generating process, especially in applications where the sensitive attribute can be inferred by the labeler, consciously or unconsciously.  
For example, training data for an automated resume scoring system could be based upon the historical scores of resumes created by a biased hiring manager or a committee of experts. 
This source of labels could then systematically score individuals from Group $B$ as having lower resume scores, an observation noted in\ randomized real world investigations  \citep{bertrand2004emily}.

Formally, the labeling bias model is:
\begin{enumerate}
\item $m$ examples are sampled from the distribution $\mathscr{D}$. 
Thus each $x \sim \mathscr{D}$. 
\item The labels $y$ for each $x$ are generated according to the 
label process from Section \ref{truelabels} with hypothesis $h^*=(h_{A}^*, h_{B}^*)$ and $\eta$. 
\item For each pair $(x,y)$, if $x \in B$ and $y=1$, then independently with probability $\nu$,  the label of this point is flipped to negative.
\end{enumerate}

This process is one-sided, so true positives become negatives in the biased training data, 
so apparent negatives becomes over-represented.  We are making a conceptual distinction that the \textit{true} labels (Step 2) 
are those generated by the original label model and these examples that get flipped
by the bias process (Step 3) are not really negative, instead they are just mislabeled. 

As $\nu$ increases more and more of the individuals in the minority group
appear negative in the training data.
Once the number of positive samples is smaller than the number of negative samples above the decision surface $h_{B}^*$, then the optimal unconstrained classifier (according to the biased data) is to simply classify all those points as negative.
\subsection{Under-Representation Bias and Labeling Bias} \label{ilfprime}
We now consider a more general model that combines Under-Representation Bias and Labeling Bias, and moreover we allow either positives or negatives of Group B (or both) to be under-represented. 
Specifically, we now have {\em three} parameters: $\beta_{POS}$, $\beta_{NEG}$, and $\nu$. Given $m$ examples drawn from $P_{\mathscr{D},r}(h^*, \eta)$, we discard each positive example of Group B with probability $1-\beta_{POS}$ and discard each negative example of Group B with probability $1-\beta_{NEG}$ to model the Under-Representation Bias. 
Next, each positive example of Group B is mislabeled as negative with probability $\nu$ to model the Labeling Bias. 
Note that the under-representation comes first: $\beta_{POS}$ and $\beta_{NEG}$ represent the probability of {\em true} positive and {\em true} negative examples from Group B staying in the data-set, respectively, regardless of whether they have been mislabeled by the agent's labelers.
\subsection{Fairness Interventions} 
Now we introduce several fairness interventions and define a notion of successful recovery from the biased training distribution. 

We consider multiple fairness constraints to examine whether the criteria have different behavior in different bias regimes.
The fairness constraints we focus on are \textbf{Equal Opportunity, Equalized Odds, and Demographic Parity}.

\begin{defn}
Classifier $h$ satisfies  \textbf{Equal Opportunity} on data distribution $\mathscr{D}$ \citep {eodds} if 
\begin{align}
P_{(x,y) \sim \mathscr{D} }(h(x)=1 | y=1, x \in A) =  P_{(x,y) \sim \mathscr{D} } (h(x)=1 | y=1, x \in B) \label{Equal Opportunity} 
\end{align}
\end{defn}
This requires that the true positive rate in Group $B$ is the same as the true positive rate in Group $A$. 

\textbf{Equalized Odds} is a similar notion, also introduced in \citep{eodds}. In addition to requiring Line \ref{Equal Opportunity}, Equalized 
Odds also requires that the false positive rates are equal across both groups.
Equivalently, we can define \textbf{Equalized Odds}
as $h \perp A | Y$, meaning that $h$ is independent of the sensitive attribute, conditioned on the true label $Y$.
We also consider \textbf{Demographic Parity} := $P(h(x)=1 | x \in A) = P(h(x)=1 | x \in B)$ \citep{dwork12}.

For each of these criteria, the overall training procedure is solving a constrained ERM problem.\footnote{We do not consider methods for efficiently solving the constrained ERM problem.}

We also consider \textbf{data Re-Weighting}, where we 
change the training data distribution to correct for the bias process and then do ERM on the new distribution. 
The overall gist of how the training data becomes biased in our models is that the positive samples from Group $B$
are under-represented in the training data so we can intervene by up-weighting the observed fraction of positives 
in the training data from Group $B$ to match the fraction of positives from the Group $A$ training data. 

In the training process we only have access to samples from the training distribution and thus when
using a fairness criterion to select among models \emph{we check the requirement on the biased training data}.

The last fairness intervention we consider is \textbf{Calibration}. 
Calibration \citep{flores, dieterich2016compas,chouldechova,pleiss2017fairness} requires that when interpreted as probabilities, the same score communicates the same information for individuals from different demographic groups.
Specifically, in the bucket of individuals receiving score $s$, the same fraction in both demographic groups is in fact truly positive. We focus on Calibration for the case of our binary classifier where there are only two scores, e.g. the scores $0$ and $1$, so in order for classifier $h=(h_A, h_B)$ to satisfy Calibration, the following equalities must hold.\footnote{If one of the conditioned-on events never occurs, such as a classifier that never classifies anyone from Group B as positive, we treat the associated equality as satisfied.}
\begin{align*}
& P_{x \sim \mathscr{D}_A}(y=1 | h_A (x)=1 ) = P_{x \sim \mathscr{D}_B}( y = 1 | h_B (x)=1)   \\
& P_{x \sim \mathscr{D}_A}(y=1 | h_A (x)=0 ) = P_{x \sim \mathscr{D}_B}( y = 1 | h_B (x)=0) 
\end{align*}
%
While the other fairness criteria are vigorously debated, Calibration is less contested as an important desiderata of machine learning models. 
Calibration has been used to defend the epistemic validity of risk prediction instruments 
\citep{flores,dieterich2016compas} and it is claimed that mis-calibrated classifiers may have serious harms and 
induce undesirable behavior when scores are used by a human actor \citep{pleiss2017fairness}.

Observe that in our model of label generation, the Bayes Optimal Classifier on the true distribution is the $h^*$ used to generate the labels initially, regardless of the values of $\eta$ and $r$.
Thus our goal for the learning process is to 
recover the original optimal classifier $h^*$, subject to training data from a range of bias models and the true label process with $(x,y) \sim P_{\mathscr{D},r}(h^*, \eta)$.
A more effective learning method would recover $h^*$ in a wider range of the model parameters (the parameters that characterize the bias process and the true label process). 
Accordingly we define \textbf{Strong-Recovery}$(R,\eta)$:
\begin{defn}
A Fairness Intervention in bias model $B$ satisfies 
Strong-Recovery$(r_0 , \eta_0)$ if for all $\eta \in [0,\eta_0)$ and all $0 < r < r_0$, when given data
corrupted by bias model $B$, the training procedure recovers
the Bayes Optimal Classifier $h^*$, given sufficient samples, for all $\beta_{POS}, \beta_{NEG} \in (0,1]$, $\nu \in [0,1)$, and $p \in (0,1]$.
\end{defn}

\section{Recovery Behavior Across Bias Models}
There are two failure modes for learning a fairness constrained classifier that we will need to be concerned with. 
First, the Bayes Optimal Hypothesis may not satisfy the fairness constraint evaluated on the biased data.
Second, within the set of hypotheses satisfying the fairness constraint,
another hypothesis (with higher error on the true distribution) may have lower error than the Bayes Optimal Classifier $h^{*}$ on the biased data.
We now describe how the multiple fairness interventions provably avoid or fail to avoid these pitfalls in increasingly complex bias models. We defer formal proofs to Section \ref{overview}.

\subsection{Under-Representation Bias} \label{sec:underrep}
Equal Opportunity and Equalized Odds both perform well in this bias model and avoid both failure modes, subject to an identical constraint on the bias and demographic parameters.

First, from the definition of the Under-Representation Bias model, observe that $h^*$ satisfies both fairness notions on the biased data, so the first failure mode does not occur.

Second, Equal Opportunity intuitively prevents the failure mode where a hypothesis is produced that appears better than $h^*$ on the biased data, such as classifying all examples from Group $B$ as negative, by forcing the two classifiers to classify the same fraction of positive examples as positive.  So, 
if we classify all the examples from Group B as negative, we have to do the same with Group A, inducing large
error on the training data from the majority Group A. 
In particular, so long as the fraction $r$ of total data from Group B is not too large and $\eta$ is not too close to $1/2$, this will not be a worthwhile trade-off for ERM (saying negative on all samples will not have lower perceived error on the biased data than $h^{*}$) and so it will not produce this outcome.

A formal proof of correctness is given in Section \ref{maintheoremproof}.
Specifically, we prove that Equal Opportunity strongly recovers from  Under-Representation Bias so long as 
\begin{align}
(1-r)(1-2\eta) +r((1-\eta)\beta-\eta) > 0 \label{Equal Opportunityineq}
\end{align}
Note that this is true for all $\eta<1/3$ and $r \in (0,1/2)$, so we have that Equal Opportunity satisfies Strong-Recovery($1/2$,$1/3)$ from 
Under-Representation Bias.  
Alternatively, we see that if $r = 1/4$ then the inequality simplifies to at least $3/4 (1-2 \eta) - \eta/4 = 3/4-(7/4) \eta$ 
so we have Strong-Recovery$(1/4,3/7)$.
Equalized Odds also recovers in this bias model with the same conditions as Equal Opportunity. 

\begin{figure}[ht]
\label{regions}
  \centering
  \includegraphics[width=0.75\textwidth]{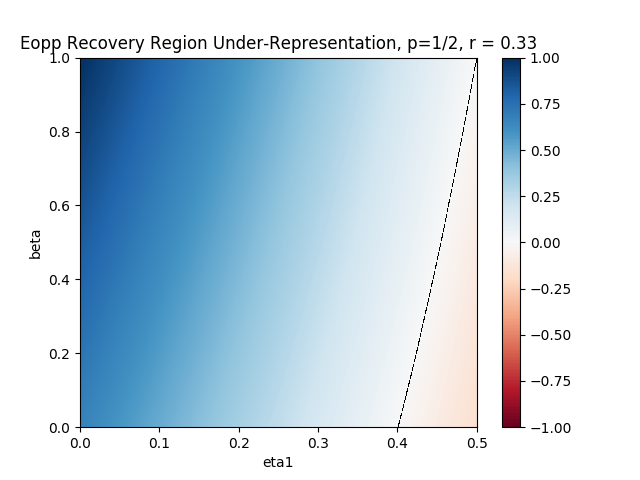}
  \caption{This figure indicates the parameter region such that Equal Opportunity Constrained ERM recovers $h^*$
  under the Under-Representation Bias Model and is a visualization of Equation \ref{Equal Opportunityineq}. $r=1/3$ and $p=1/2$. We label each pair $(\eta, \beta)$ with blue if it satisfies the inequality and red otherwise. This plot shows how smaller $\eta$ means we can recover from lower $\beta$. 
  Blue means $h^*$ is recovered.
  The dashed black line indicates the boundary between recovering $h^*$ and failing to recover $h^*$.}
\end{figure}

In contrast, Demographic Parity fails to recover $h^*$ even if $\eta=0$.
If $p=1/2$, $\eta=0$, and $\beta=1/2$ and we originally had $n$ samples, then
the Bayes Optimal Classifier does not satisfy Demographic Parity on the biased data
since the fraction of samples that will be labelled positive is $\frac{1}{3} \neq \frac{1}{2}$. 

Similarly, if we let 
$\eta \neq 0 , \beta < 1 $,
then in order to match the fraction of positive classifications made by $h_{A}^{*}$,
$h_{B}$ is forced to classify a larger region of the input spaces as positive than $h_{B}^*$ would in the absence of biased data and so we do not 
recover $h_{B}^{*}$.

Another way to intervene in the Under-Representation Bias model
would just be to re-weight
 the training data to account for the under-sampling of positives from Group $B$.
If we really know positives from Group $B$ are under-represented, we can change our objective function 
$ min \sum_{i=1}^{m} I(h(x) \neq y) $
 by changing each indicator function such that minimizing the sum of indicators measures the loss on the true distribution 
 and not the loss on the biased training distribution. 

Define $B^{+} = \{ x \in B \; s.t.\; y=1 \}$. 
Then let, 
\[
I'(h(x),y) = 
\begin{cases}
\frac{1}{\beta} \quad h(x) \neq 1 \quad and \quad x \in B^{+} \\
0 \quad  h(x) = 1 \quad and \quad x \in B^{+} \\
I(h(x) \neq y) \quad otherwise 
\end{cases}
\]
Then we use this new indicator in the objective function.  
This new loss function is an unbiased estimator of the true unbiased risk, so uniform convergence on this estimator will suffice to learn $h^*$.
We can infer the value of $\beta$ from the data for Group A if we know the data from Group B is corrupted by this bias model.  One concern  with re-weighting in general is that the functional form of the correction is tied to the exact bias model.

As we show 
in Section \ref{calib}, Calibration has strange results in this bias model. 
Specifically, when the bias is such that ERM fails to recover $h^*$ (i.e when $(1-\eta)\beta < \eta$), then the Calibration constraint can only be satisfied by a trivial classifier that assigns all of Group $A$ to one label and all of Group $B$ to the alternative label. 
For typical parameters, this will result in Group $B$ being given the negative label and Group $A$ will be
assigned as all positive.
This will not recover $h^*$ and is in fact substantially worse than merely using ERM. 
Un-constrained ERM would learn badly on Group $B$ but would recover $h_{A}^{*}$ for Group $A$. 

When the bias regime is such that $(1-\eta)\beta > \eta$, plain ERM recovers $h^*$, while enforcing Calibration will lead to excess true error 
on both demographic groups over the true error of $h^{*}$.
In particular, satisfying Calibration on the biased data requires
intentionally classifying some negative input space from Group $A$ as positive and classifying some positive input space from Group $B$ as negative.
These results suggest that Calibration is an actively harmful intervention (for both groups) in our model, when compared to plain ERM, across all model parameters.

 In summary, for the Under-Representation Bias model, 
 the fairness interventions Equalized Odds, Equal Opportunity, and Re-Weighting recover $h^*$ under a range of parameters. 
 However, Demographic Parity is inadequate even for $\eta=0$ and will not recover $h^*$ for non-vacuous bias parameters. 

\subsection{Labeling Bias} \label{labelbias}
In Section \ref{overview}, we prove that Equal Opportunity constrained ERM on 
data biased by the Labeling Bias model
 also finds the Bayes-Optimal Classifier, under similar parameter conditions to the previous bias model.

Interestingly, in contrast to Under-Representation Bias, \textit{Labeling Bias cannot be corrected} 
by \textit{Equalized Odds}.  
The problem is the first failure mode. 
For example, consider $\eta =  0$ but where $\nu \neq 0$.
The Bayes Optimal Classifier $h_{A}^{*}$ for Group $A$ has false positive rate of 0 and true positive rate of $1$.
However, since $\nu>0$, there is no classifier for Group $B$ that achieves both of these rates simultaneously. In particular, the only way to classify the negative individuals in the positive region as negative is for the classifier to decrease its true positive rate from $1$. Therefore, Equalized Odds rules out usage of $h_A^*$.  This violation holds for $\eta \neq 0$ as well. 

In contrast, $h^*$ does satisfy Equal Opportunity on the biased data, and given the conditions in Theorem \ref{maintheorem}, it will be the lowest error such classifier on the biased data. 

When \emph{just} Labeling Bias is present, observe that $h^{*}$ still satisfies 
Demographic Parity on the biased data, since in contrast to the Under-Representation Bias case, the positives that are flipped to negative still appear in the training data.
In this case, Demographic Parity will experience strong recovery when $(1-r)*(1-2\eta) + r((1-eta)(1-2\nu)-\eta) > 0$. 
This inequality  is a simple variation of the first inequality in Theorem \ref{maintheorem}, and a simplification of that proof will yield this result, if the only present bias is Labeling Bias.

The Re-Weighting intervention 
is to change the weighting of observed positives in the training data for Group $B$
so that we have the same  fraction of positives in Group $B$ as in Group $A$.
Define $p_{A,1}:= $ the fraction of positive individuals in Group $A$ and $p_{B,1}:=$ the \emph{observed} fraction of positives in $B$ in the biased data. $p_{A,0}$ and $p_{B,0}$ refer to the observed fraction of negative individuals in Group $A$ and Group $B$ in the biased data.

We need a re-weighting factor $Z$ such that:
\begin{align*}
 \frac{p_{A,1}}{p_{A,0}} &= \frac{Z p_{B,1} }{p_{B,0}} \\
  \frac{p_{A,1}}{1-p_{A,1}} &= \frac{Z p_{A,1}(1-\nu)}{p_{A,0} + p_{A,1}\nu} = \frac{Z p_{A,1}(1-\nu)}{1-p_{A,1} + p_{A,1} \nu } \\
 Z &= \frac{1-p_{A,1}(1-\nu) }{(1-\nu)(1-p_{A,1})}  
\end{align*}
We prove in Section \ref{reweightarg} that this correction factor will lead to the positive region of $h_{B}^{*}$ having a higher weight of positive examples than negative examples and simultaneously the negative region of $h_{B}^{*}$ having a higher weight of negative examples than positive examples. 
This causes ERM to learn the optimal hypothesis $h^{*}$.  We can infer the value of $\nu$ by comparing the fraction of positives in Group $A$ and Group $B$. 

In summary, Equal Opportunity, Demographic Parity, and Re-Weighting Interventions recover well in this bias model (Labeling Bias)
while Equalized Odds is inadequate.

\subsection{Under-Representation Bias and Labeling Bias} 
In this most general model that combines the two previous models, Re-Weighting the data is now no longer sufficient to recover the true classifier.  For example, consider the case where $\eta=0$ and $p=1/4$, $\nu=1/2$ and $\beta_{NEG}=1/3$ and $\beta_{POS} = 1 $. 
If there were $n$ points originally from group $B$, 
then in expectation $3n/4$ were negative and $n/4$ were positive.
After the bias process, in expectation there are $n/4$ negatives on the negative side of $h^*$, 
and on the positive side of $h^*$ we have $n/8$ correctly labelled positives and what appear to be $n/8$ negative samples.

The Re-Weighting intervention will not do anything in expectation because the overall fractions are still correct; we have $n/2$ total points
with one quarter of them labeled positive.
ERM is now indifferent between $h^*$ and labeling all samples from Group $B$ as negative. 
If we just slightly increase the parameter $\nu$ and reduce $\beta_{POS}$ then in expectation ERM will strictly prefer labeling
all the samples negatively. 

While the Re-Weighting  method fails, we prove that Equal Opportunity-constrained ERM recovers the Bayes Optimal Classifier $h^*$ as long as we satisfy a condition ensuring that Group A has sufficient mass and the signal is not too noisy.
 As with the previous models, Demographic Parity and Equalized Odds 
are not satisfied by $h^*$ on minimally biased data and so they will not recover the Bayes-Optimal classifier.

\section{Main Results} \label{overview}
We now present our main theorem formally.
Define the biased error of a classifier $h$ as its error rate computed on the biased distribution. 
\begin{thm} \label{maintheorem}
Assume true labels are generated by  $P_{\mathscr{D},r}(h^*, \eta)$ corrupted by both Under-Representation bias and Labeling bias with parameters $\beta_{POS}, \beta_{NEG},\nu$, and assume that
\begin{align}
 (1-r)(1-2\eta) + & r( (1-\eta)\beta_{POS} (1-2\nu) - \eta \beta_{NEG}) > 0 \label{proofc1} \\
& \quad  \text{and} \nonumber \\
(1-r)(1-2\eta) + & r ((1-\eta) \beta_{NEG} -(1-2\nu )\beta_{POS} \eta ) > 0 \label{proofc2}
\end{align}

Then $h^*=(h_{A}^{*}, h_{B}^{*})$ is the lowest biased error 
classifier satisfying Equality of Opportunity on
the biased training distribution and thus $h^*$ is recovered by Equal Opportunity constrained ERM. 

Note $\beta_{POS}, \beta_{NEG} \in (0,1]$, $\nu \in [0,1)$, $\eta \in [0,1/2)$, $r \in (0,1)$ and $p \in (0,1]$.
Condition \ref{maintheorem} refers to Equation \ref{proofc1} and Equation \ref{proofc2}.
\end{thm}

This case contains our other results as special cases and in the next section we prove our main theorem in this bias model.
Note that if Equation \ref{proofc1} is not satisfied then the all-negative hypothesis will have the lowest biased error among hypotheses satisfying
Equal Opportunity on the biased training distribution. 
Similarly, if Equation \ref{proofc2} is not satisfied then the all-positive hypothesis will have the lowest biased error among hypotheses satisfying
Equal Opportunity on the biased training distribution. 
Thus Theorem \ref{maintheorem} is tight. 

To give a feel for the formula in Theorem \ref{maintheorem}, note that the case of small $r$ 
is {\em good} for our intervention, because the advantaged Group $A$ is large enough to pull the classification of the disadvantaged Group $B$ in the right direction. For example, if $r \leq \frac{1}{3}$ then the bounds are satisfied for all $\eta < \frac{1}{4}$ (and if $r \leq \frac{1}{4}$ then the bounds are satisfied for all $\eta < \frac{1}{3}$) for {\em any} under-representation biases $\beta_{POS},\beta_{NEG}>0$ and {\em any} labeling bias $\nu<1$.

Thus, Equal Opportunity Strongly Recovers with $(1/4,1/3)$ and $(1/3,1/4)$ in the Under-Representation and Labeling Bias model. 

Table \ref{summarytab} summarizes the results in the core interventions and the three core bias models.
The contents of each square indicate if recovery is possible in a bias model with an intervention 
and what constraints need to be satisfied for recovery.

\begin{table}
    \begin{tabular}{ |p{2cm}|p{3cm}|p{3cm}|p{3cm}| }
    \hline
	Intervention   & Under-Representation & Labeling Bias & Both   \\ \hline
	Equal Opportunity-ERM & Yes: $(1-r)(1-2\eta) +r((1-\eta)\beta-\eta) > 0$
 & Yes:  $(1-r)(1-2\eta) + r ( (1-\eta)(1-2\nu) -  \eta) > 0 $ & Yes: Using Condition \ref{maintheorem} \\ \hline
	Equalized Odds & Yes: $(1-r)(1-2\eta) +r((1-\eta)\beta-\eta) > 0$ & No & No \\ \hline
	Re-weighting Class B: & Yes & Yes & No \\
    \hline
   Demographic Parity: & No & Yes & No \\
   \hline
    \end{tabular}
    \caption{Summary of recovery behavior of multiple fairness interventions in bias models.}
    \label{summarytab}
\end{table}

\comment{
\begin{figure}[H]
\centering
\begin{minipage}{.5\textwidth}
  \centering
  \includegraphics[width=\textwidth]{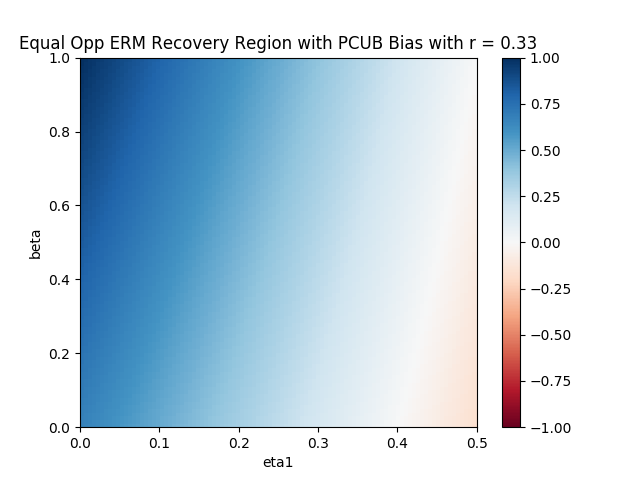}
  \label{fig:test1}
\end{minipage}%
\begin{minipage}{.5\textwidth}
  \centering
  \includegraphics[width=\textwidth]{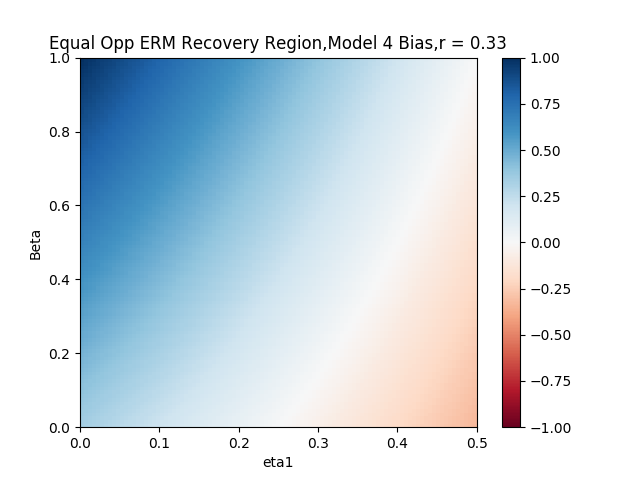}
  \label{fig:test2}
\end{minipage}
\caption{  
The left hand figure is the parameter region where Equal Opportunity ERM recovers $h*$ under the PCUB bias model.
Blue means the inequality is satisfied.
The right hand figure plots the region for EO-ERM but for the Bias Model PCUB+ILF. 
}\end{figure}
}

\subsection{Proof of Main Theorem} \label{maintheoremproof}
In this section we present the proof of the main result, \textbf{Theorem} \ref{maintheorem}.
We want to show that the lowest biased error classifier satisfying Equal Opportunity on the biased data is $h^*$, given Condition \ref{maintheorem}.

The first step of the proof is to show that $h^*$ satisfies Equal Opportunity on the biased training data.
Note: the lemmas and claims here are all in the Under-Representation  Bias combined with Labeling Bias Model, the most general bias model. 
\begin{lem}\label{lem:bayesopt}
$h^*=(h^*_{A}, h_{B}^*)$ satisfies Equal Opportunity on the biased data distribution. 
\end{lem}

\begin{proof}
First, let's consider the easiest case with $\eta=0$, $\beta_{POS}=\beta_{NEG}=1$, and $\nu=0$.  
Recall that $h^*$ is the pair of classifiers used to generate the labels. 
When $\eta=0$, $h^*$ is a perfect classifier for both groups so Equal Opportunity is trivially satisfied.  
Now, let's consider arbitrary $0 \leq \eta < 1/2$.  
Recall that $p=Pr_{D_A}(h_{A}^*(x) = 1 | x \in A) = Pr_{D_B}(h_{B}^*(x) = 1 | x \in B) $.

By our assumption that Group A and Group B have equal values of $p$ and $\eta$ we have 
\[ \Pr(h^*_A(x)=1|Y=1, x \in A) = \frac{p(1-\eta)}{p(1-\eta)+(1-p)\eta} = \Pr(h^*_B (x) =1|Y=1, x \in  B) \]

Next consider when we have both Under-Representation Bias and Labeling Bias.
Recall that $\beta_{POS}, \beta_{NEG}>0$ is the probability that a positive or negative sample from Group $B$ is \emph{not filtered} out of the training data
while $\nu<1$ is the probability a positive label is flipped and this flipping occurs after the filtering process. 
Then,
\begin{align*}
& \{ \text{True Positive Rate on Group A} \} := \Pr(h^*_A(x)=1|Y=1, x \in A) = \\
&\frac{p(1-\eta)}{p(1-\eta)+(1-p)\eta} =
\frac{p(1-\eta)\beta_{POS}(1-\nu)}{p(1-\eta) \beta_{POS}(1-\nu)+(1-p)\eta \beta_{POS}(1-\nu)}   \\
&= \Pr(h^*_B (x) =1|Y=1, x \in  B) := \{ \text{True Positive Rate on Group B} \}
\end{align*}
so Equal Opportunity is still satisfied.  

In words, the bias model removes or flips positive points from Group $B$ independent of their 
location relative to the optimal hypothesis class. 
Thus positive points throughout the input space are
are equally likely to be removed, so the overall probability of true positives being classified as positives is not changed.
\end{proof}



Now we describe how a candidate classifier $h_{B}$ differs from $h^*_{B}$.
We can describe the difference between the classifiers by noting the regions in the input space that each classifier gives a specific label. 
This gives rise to four regions of interest with probability mass as follows:
\begin{align*}
    & p_{1B} = P_{1B}(h_{B}):= P_{x \in \mathscr{D}_B} (h_{B}^{*}(x)=1 \land h_{B} (x) =0) \\
    &  p_{2B} = P_{2B}(h_{B}) := P_{x \in \mathscr{D}_B}(h_{B}^{*}(x)=0 \land h_{B} (x) =1) \\
    & p-p_{1B} = P_{x \in \mathscr{D}_B}(h_{B}^{*}(x)=1 \land h_{B} (x) =1) \\
    & 1-p-p_{2B} = P_{x \in \mathscr{D}_B}(h_{B}^{*}(x)=0 \land h_{B} (x) =0) 
\end{align*}
These probabilities are made with reference to the regions in input space \emph{before} the bias process. 
$p_{1B}$ and $p_{2B}$ are functions of $h_{B}$ to make explicit that there may be multiple hypotheses with different functional forms that could allocate the same amount of probability mass to parts of the input space where $h_{B}^{*}$ and $h_{B}$ agree on labeling as  positive and negative respectively.

The partition of probability mass into these regions is easiest to visualize for hyperplanes but will hold with other hypothesis classes.
$p_{1A}$ and $p_{2A}$ are defined similarly with respect to $h_{A}^*$ and $\mathscr{D}_A$. 
A schematic with hyper-planes is given in Figure \ref{regionsh*}.
\begin{figure}[ht]
  \centering
  \includegraphics[width=0.5\textwidth]{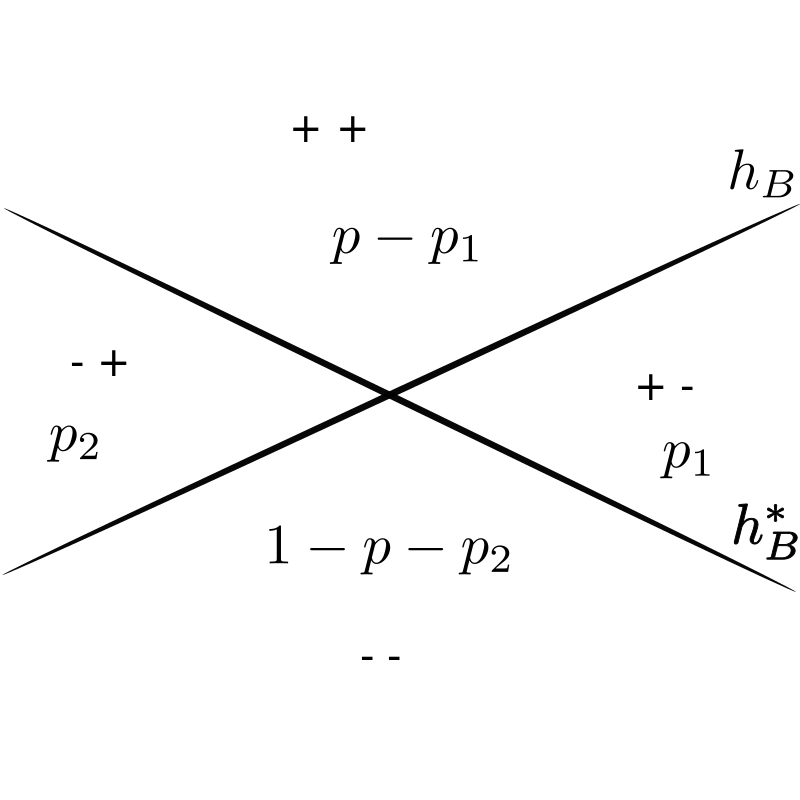}
  \caption{Differences between $h_B$ and $h_{B}^*$ \label{regionsh*}
  measured with probabilities in the true data distribution  (before the effects of the bias model).}
\end{figure}

To show that $h^*$ has the lowest error on the true distribution, we first show how given any a pair of classifiers $h_A$ and $h_B$, which jointly satisfy Equal Opportunity (Equal Opportunity) on the biased distribution, we
can transform $\{ h_{A}, h_{B} \}$ into a pair of classifiers still satisfying Equal Opportunity with at most one non-zero parameter from  $\{p_{1B}, p_{2B} \}$, and at most one non-zero parameter from $\{p_{1A}, p_{2A} \}$, while also not increasing biased error. 

The final step of our proof argues that out of the family of all hypotheses with (1) at most one non-zero parameter for the hypothesis on Group $A$, (2) at most one non-zero parameter for the hypothesis on Group $B$, (3) and jointly satisfying Equal Opportunity on the biased data, $h^*$ has the lowest biased error. 

These steps combined imply that $h^*$ is the lowest biased error hypothesis that satisfies Equal Opportunity. 

\begin{lem}
\label{shrink}
Given a pair of classifiers $h_A$ and $h_B$ which satisfy Equal Opportunity on the biased data 
we can find a pair of classifiers $h_{A}^{'}$ and $h_{B}^{'}$ satisfying 
\begin{enumerate}
    \item At most one of $\{P_{1A}(h_{A}^{'}),P_{2A}(h_{A}^{'})\}$ is non-zero and at most one of  $\{P_{1B}(h_{B}^{'}),P_{2B}(h_{B}^{'}) \}$ is non-zero.
    \item $(h_{A}^{'}, h_{B}^{'} )$ has error at most that of $(h_A, h_B)$ on the biased distribution. 
    \item $h_{A}^{'}$ and $h_{B}^{'}$ satisfy Equal Opportunity. 
\end{enumerate}
\end{lem}
\begin{proof}
We want to exhibit a pair of classifiers with lower biased error that zeros out one of the parameters. 
We do this by modifying each classifier separately, while keeping the true positive rate on the biased data fixed to ensure we  satisfy Equal Opportunity.

First, consider Group $A$ and suppose that $P_{1A}(h_{A}), P_{2A}(h_{A}) > 0 $ since otherwise we do not need to modify $h_A$.
We hold the true positive rate of $h_A$ constant and shrink $p_{2A}$ towards zero.
As we shrink $p_{2A}$, we must shrink $p_{1A}$ towards zero in order hold the true positive rate fixed (and thus satisfy Equal Opportunity). 

The un-normalized\footnote{The normalization factor for these rates for Group $A$ and Group $B$ is the same so this term can be cancelled.} True Positive Rate (constrained by Equal Opportunity) is
$(p-p_{1A})(1-\eta) + p_{2A} \eta = p(1-\eta) -p_{1A}(1-\eta) + p_{2A} \eta   = (p-p_{1B})(1-\eta) + p_{2B} \eta$. 
Since the $p(1-\eta)$ term is independent of the classifier $h_A$, keeping the true positive rate constant is equivalent to keeping $C:= -p_{1A}(1-\eta) + p_{2A} \eta $ constant. 

Let $\Delta$ be the amount we wish to shrink $p_{2A}$ and let $f(\Delta)$ be the amount we must shrink $p_{1A}$ to keep $C$ fixed. 
Then, 
\begin{align*}
& f(\Delta) = \Delta \frac{\eta}{1-\eta}
\end{align*}
We continue shrinking these parameters until either $p_{1A}$ or $p_{2A}$ has hit zero. 
To this see can be always be done, note that $p_{1A} = -C' + p_{2A} \frac{\eta}{1-\eta}$ for $C'= \frac{C}{1-\eta}$. 
Which term hits zero will depend on the sign of $C$. 
 
Observe for Group $A$ this process will clearly reduce training error since we are decreasing both $p_{1A}$ and $p_{2A}$ and the error on group $A$ is monotone increasing (and linear) with respect to $p_{1A}+p_{2A}$. 

We then separately do this same shrinking process for group $B$.
Now we show the biased error decreases for Group $B$. 
\begin{align*}
& \{ \text{Change in Biased Error after Shrinking} P_{2B}  =  \Delta[\eta \beta_{POS} (1-\nu) - (1-\eta) \beta_{NEG} - \eta \beta_{POS} \nu] \\
& \{ \text{Change in Biased Error after Shrinking} P_{1B}  =
f(\Delta)[\eta \beta_{NEG} + (1-\eta)\beta_{POS} \nu - (1-\eta) \beta_{POS}(1-\nu)]
\end{align*}
Thus the overall biased error change for Group $B$ is the sum of the two above terms and simplifies to become
\begin{multline*}
    = \Delta \eta \beta_{POS} (1-\nu) - f(\Delta) (1-\eta) \beta_{POS} (1-\nu)  \\ 
 + \Delta(- (1-\eta) \beta_{NEG} - \eta \beta_{POS} \nu)  + (\Delta)( \eta \beta_{NEG} + (1-\eta) \beta_{POS}\nu)
\end{multline*}
\comment{
\begin{align*}
& = \Delta (\eta \beta_{POS} (1-\nu) - (1-\eta) \beta_{NEG} - \eta \beta_{POS} \nu) +  f(\Delta)(  \eta \beta_{NEG} + (1-\eta)\beta_{POS} \nu  - (1-\eta) \beta_{POS}(1-\nu) ) 
 \\
  & = \Delta \eta \beta_{POS} (1-\nu) - f(\Delta) (1-\eta) \beta_{POS} (1-\nu)   
  + \Delta(- (1-\eta) \beta_{NEG} - \eta \beta_{POS} \nu)  + (\Delta)( \eta \beta_{NEG} + (1-\eta) \beta_{POS}\nu)
 \end{align*}
 }
 The first two terms vanish because of $f(\Delta)= \Delta \frac{\eta}{1-\eta}$.
 \begin{align*}
 & = \Delta(- (1-\eta) \beta_{NEG} - \eta \beta_{POS} \nu) + 
 f(\Delta)( \eta \beta_{NEG} + (1-\eta) \beta_{POS}\nu)\\ 
 & =\Delta(- (1-\eta) \beta_{NEG} - \eta \beta_{POS} \nu) + 
 \Delta\frac{\eta^2}{1-\eta} \beta_{NEG} + \Delta \eta \beta_{POS}\nu\\ 
 &= \Delta ( \frac{\eta^2}{1-\eta} \beta_{NEG} - (1-\eta) \beta_{NEG}) < 0
 \end{align*}
Since this term is negative, we have shown that this modification process decreases error on the biased training data for both Group $A$ and Group $B$ while keeping the true positive rate fixed.
\end{proof}

\begin{lem} \label{delta}
If we have two classifiers $h_A$ and $h_B$ each with one non-zero parameter and the classifiers satisfy the Equal Opportunity constraint, then $p_{1B} = p_{1A} $ and $p_{2B}=p_{2A}$.
\end{lem}

\begin{proof}
Recall that the Equal Opportunity constraint requires that these expressions be equal.
\begin{align*}
    & (p-p_{1A})(1-\eta)  + p_{2A} \eta = (p-p_{1B})(1-\eta) + p_{2B} \\
    &  p_{2A} \eta - p_{1A} (1-\eta) =  p_{2B} \eta - p_{1A}(1-\eta)
\end{align*}
Then the theorem follows from inspecting the second equality.  
\end{proof}
This lemma makes explicit that when the classifiers each have only one non-zero parameter and satisfy Equal Opportunity, then the non-zero parameter corresponds to the same region. 

\begin{lem}
Of hypotheses satisfying
 ($p_{1A}=p_{1B}$ and $p_{2A}=p_{2B}=0$) or 
($p_{1A}=p_{1B}=0$ and $p_{2A}=p_{2B}$), if these inequalities hold:
\begin{align*}
 (1-r)(1-2\eta) + & r( (1-\eta)\beta_{POS} (1-2\nu) - \eta \beta_{NEG}) > 0 \\
& \quad  \text{and} \\
 (1-r)(1-2\eta) + & r ((1-\eta) \beta_{NEG} -\eta \beta_{POS} (1-2\nu ) ) > 0 
\end{align*} 
then the lowest biased error classifier satisfying Equal Opportunity on the biased data is $h^*=(h_{A}^{*}, h_{B}^{*})$.
\end{lem}

\begin{proof}
First, we sketch the proof informally. 
Consider three cases which depend on how the bias process affects the unconstrained optimum for Group $B$ on the biased data. 
In the first case, in the biased data distribution, the region $X^{+} :=  \{ x \; s.t. \;  h_{B}^{*} (x) = 1 \}$ has more positive than negative samples in expectation and the region $X^{-} :=  \{ x \; s.t. \;  h_{B}^{*} (x) = 0 \}$ has more negative than positive samples in expectation.
 In the second case, there are more positive than negative samples throughout the entire input space  in the biased distribution.
 In the third and final case, there are more negative than positive samples throughout the input space in the biased distribution. 

In these three cases, the optimal hypothesis is exactly one of $\{ h_{B}^*, h_{B}^{1}, h_{B}^{0} \}$, respectively.  
The second two hypotheses mean labelling all inputs as positive and labelling all inputs as negative, respectively.
These three hypotheses correspond to hypotheses with at most one non-zero parameter. 

For instance,  $h_{B}^{1}$ occurs when $p_{2B}=1-p$ and $p_{1B}=0$.
Each of the three hypotheses occur when the one non-zero parameter attains a location on the boundary of its range of values.
When $p_{2B}$ is allowed to be non-zero, if instead $p_{2B}=0$ (and thus it also must be that $p_{1B}=0$), the hypothesis is equivalent to $h_{B}^{*}$. 
A similar relationship holds for $h^{0}$ and $p_{1}$. 

In order to show the theorem, we prove that if $h^*$ has lower biased error than $h^{1}=(h_{A}^{1}, h_{B}^{1})$ and $h^{0}=(h_{A}^{0}, h_{B}^{0})$ on the biased data distribution, then $h^{*}$ has the lowest error among all hypotheses with at most one non-zero parameter and satisfying Equal Opportunity.

To see this, consider $h_A$ and $h_B$ with the same non-zero parameter equal to $\Delta$.
Then the error of $h_A$ is a linear function of $\Delta$. 
Similarly, the error of $h_B$ is a linear function of $\Delta$.
The overall error of $h=(h_A, h_B)$ is a weighted combination of  the error of $h^*$ and the error of $h^{0}$ or $h^{1}$,
so the overall error of $h$ is thus linear in $\Delta$, so the optimal hypothesis parametrized by $\Delta$ must occur on the boundaries of the region of $\Delta$, so the optimal hypothesis is one of $\{h^{*} , h^{0}, h^{1} \}$.
We then show that the inequalities we assume in the theorem enforce that $h^*$ has strictly lower error than
$h^{0}$ or $h^{1}$. 
Formally, we enumerate the possible events: 
\begin{center}
    \begin{tabular}{ |l | l | l | l | p{5cm} |}
    \hline
      Type & Sign of  $h^*$ &  Label in Biased Data & Un-Normalized Probability of Event \\ \hline
      A    & +    & +     & $R_1 = (1-r)p(1-\eta)$ \\ \hline 
      A    & +    & -     & $R_2 = (1-r)p\eta$ \\ \hline 
      A    & -    & +     & $R_3 = (1-r)(1-p)\eta$ \\ \hline 
      A    & -    & -     & $R_4 = (1-r)(1-p)(1-\eta)$ \\ \hline 
      B    & +    & +     & $R_5 = rp(1-\eta)\beta_{POS}(1-\nu)$ \\ \hline 
      B    & +    & -     & $R_6 = rp[(1-\eta)\beta_{POS}\nu+\eta \beta_{NEG}]$ \\ \hline 
      B    & -    & +     & $R_7 = r(1-p)(\eta \beta_{POS})(1-\nu)$ \\ \hline 
      B    & -    & -     & $R_8 = r(1-p)[(1-\eta)\beta_{NEG} + \eta \beta_{POS} \nu ]$ \\ \hline 
    \end{tabular}
\end{center}
The probabilities on the far right hand side are not normalized. 
First we show that the $err(h^*) < err(h^{1})$.
$err(h^*)= R_{2}+R_{3}+R_{6}+R_{7}$ and $err(h^{1})=R_{2}+R_{4}+R_{6}+R_{8}$, thus
$err(h^*) < err(h^1)$ if and only if $R_{3} + R_{7} < R_{4} + R_{8}$ or thus if
\begin{multline*}
 (1-r)(1-p)\eta + r(1-p)(\eta \beta_{POS})(1-\nu)  \\ 
 < (1-r)(1-p)(1-\eta) + r(1-p)[(1-\eta)\beta_{NEG} + \eta \beta_{POS} \nu ] 
\end{multline*}
Equivalently,
\begin{align}
& 0 < (1-r)(1-2\eta) + r [(1-\eta)\beta_{NEG} - \eta \beta_{POS} (1-2\nu) ] \label{const1}
\end{align}
Now we consider $h^*$ compared to $h^{0}$.
Then  
$err(h^{0})=R_{1}+R_{3}+R_{5}+R_{7}$
Then $err(h^*) < err(h^{0})$ if and only if $R_{2} + R_{6} <R_{1} + R_{5}$. 
\begin{align*}
 (1-r)p\eta+ rp[(1-\eta)\beta_{POS}\nu+\eta \beta_{NEG} ] < (1-r)p(1-\eta) + rp(1-\eta)\beta_{POS}(1-\nu) 
 \end{align*}
 Equivalently,
 \begin{align}
 0 < (1-r)(1-2\eta) + r((1-\eta)\beta_{POS}(1-2\nu) - \eta \beta_{NEG}) \label{const2}
\end{align}
Thus we have shown that the error of $h^*$ is less than the error of of $h^{1}$ and
$h^{0}$ if and only if both Lines \ref{const1} and \ref{const2} are true, which we
assume in our theorem.

Now we show that we error of $h=(h_A, h_B)$ is linear in $\Delta$.
There are two cases depending on what parameter of $h$ is non-zero.

Let $h$ be a hypothesis such that $P_{1A}(h_{A} ) = p_{1B} = \Delta$ and $P_{2A}(h_A ) = p_{2B} = 0$ and
$\Delta \in [0,p]$.
\begin{align*}
& err(h) = R_1 \frac{\Delta}{p} + R_{2} \frac{p-\Delta}{p}+R_{3} + R_5 \frac{\Delta}{p} + R_6 \frac{p - \Delta}{p} +R_7 \\
& = \frac{\Delta}{p} err(h^{0}) + \frac{p-\Delta}{p} err(h^*) 
\end{align*}
On the other case let $P_{1A}(h_{A} ) = p_{1B} = 0$ and $P_{2A}(h_A ) = p_{2B} = \Delta$ and $\Delta \in [0,1-p]$.
\begin{align*}
    & err(h)=R_2 +  \frac{1-p-\Delta}{1-p}R_3 +\frac{\Delta}{1-p} R_4 + R_6 + \frac{1-p-\Delta}{1-p}R_7  + \frac{\Delta}{1-p} R_8 \\
    & =\frac{\Delta}{1-p} err(h^{1}) + \frac{1-p-\Delta}{1-p} err(h^*) \\
\end{align*}
Thus the error of $h$ is linear in $\Delta$ and boundary values for $\Delta$ correspond to the hypotheses in $\{ h^{*}, h^{0}, h^{1} \}$.
These two arguments show that:
\begin{enumerate}
    \item Any single parameter $h$ is a weighted sum of ($h^*$ and $h^{0}$ ) or is a weighted sum of ($h^*$ and $h^{1}$) and so is linear in $\Delta$. 
    The boundary values of $\Delta$ correspond to $\{h^{*}, h^{0}, h^{1} \}$.
\item Since the optimal value of a linear function occurs on the boundaries of its range, the optimal Equal Opportunity classifier with at most one non-zero parameter is one of $\{h^{*}, h^{0}, h^{1} \}$.
\item  The inequalities in the theorem statement enforce that $h^*$ has lower biased error than either $h^0$ or $h^1$, so $h^*$ has the lowest biased error of any single parameter hypothesis satisfying Equal Opportunity. 
\end{enumerate}
\end{proof}
If the conditions in the Theorem \emph{do not hold}, then $h^{*}$ will not have lower error than $h^{0}$ and $h^{1}$.  

\subsection{Verification Re-Weighting Recovers from Labeling Bias}  \label{reweightarg}

The way we intervene by Reweighting is we multiply the loss term for mis-classifying positive examples in Group $B$ by a factor $Z$  such that the weighted fraction of positive examples in biased data for Group $B$ is the same as the overall fraction of positive examples in Group $A$. 

The goal of this reweighting is to ensure that the ratio of positive to negative samples
in the positive region of $h_{B}^*$ is greater than $1$ while the ratio is less than $1$ in the negative region of $h_{B}^*$.
Thus the re-weighted probabilities need to simultaneously satisfy:
\begin{align*}
& \frac{P(y=1|h_{B}^*(x)=1)}{P(y=0 | h_{B}^*(x) =1 ) } = \frac{Z[(1-\eta )(1-\nu )] }{ (\eta + (1-\eta)\nu)}  > 1   \\
& \frac{P(y=1 |h_{B}^*(x)= 0 )}{P(y= | h_{B}^*(x) = 0  ) } = \frac{Z[\eta (1-\nu)]}{((1-\eta) + \eta \nu)} < 1\\
\end{align*}
 The two constraints are equivalent to requiring that:
\begin{align} \label{ilfweight}
& \frac{\eta + (1-\eta) \nu }{ (1-\eta)(1-\nu)} < Z < \frac{1-\eta + \eta \nu }{\eta (1-\nu)} 
\end{align} 
Recall from Section \ref{labelbias} that  $Z = \frac{1-P_{A,1}(1-\nu)}{(1-\nu)(1-P_{A,1})}$

First we show the right hand inequality.
\begin{align*}
    &\frac{1-p_{A,1}(1-\nu)}{(1-\nu)(1-p_{A,1})} < \frac{1-\eta + \eta \nu }{\eta (1-\nu)} \\
    & 0 <  \frac{1-\eta + \eta \nu }{\eta } - \frac{1-p_{A,1}(1-\nu)}{(1-p_{A,1})} \\
\end{align*}
Observe that both terms are linear in $\nu$. When $\nu = 0$, the inequality becomes $\frac{1-\eta}{\eta}- \frac{1-p_{A,1}}{1-p_{A,1}} = \frac{1-\eta}{\eta} - 1 > 0 $. 
In our bias model $\nu \in [0,1)$, but if  $\nu = 1$, the inequality becomes $\frac{1}{\eta}-\frac{1}{1-p_{A,1}} > 0 $.
Thus Line \ref{ilfweight} holds if both $\frac{1-\eta}{\eta} -1 > 0$ and
$\frac{1}{\eta}-\frac{1}{1-p_{A,1}} > 0$. 

$\frac{1-\eta}{\eta} -1 > 0$ is clearly true because $0 < \eta < 1/2 $.

To see that $\frac{1}{\eta}-\frac{1}{1-p_{A,1}} > 0$, note that this is equivalent to $\eta<1-p_{A,1}$, where the right-hand-side is the overall fraction of negative examples in $A$.  This is clearly true because the positive region of $h_A^*$ has exactly an $\eta$ fraction of negatives, and the negative region of $h_A^*$ has a $1-\eta>\eta$ fraction of negatives.
\comment{
Observe if $\eta < 1-p_{A,1}$, then $\frac{1}{\eta}-\frac{1}{1-p_{A,1}} > 0 $. 
To see that $\eta < 1 - p_{A,1}$:
\begin{align*}
& \eta < 1 - p_{A,1} = 1-(p(1-\eta)+(1-p)\eta) = 1- (p+\eta -2p\eta) = 1- p - \eta +2p\eta \\
& \eta < 1-\eta + p(2\eta-1) \\
& p(1-2\eta) < 1-2\eta
\end{align*}
The final inequality clearly holds since $p \in (0,1]$. 

Thus we have shown that the right hand inequality in Line \ref{ilfweight}
is satisfied for $\nu$ on its boundaries and  since the inequality is linear in 
$\nu$, clearly the inequality holds for for $\nu \in [0,1)$. 
Thus this constraint is satisfied for all parameters in our bias model.
}

Now we show the left hand inequality in Line \ref{ilfweight}. 
\begin{align} 
& \frac{\eta + (1-\eta) \nu }{ (1-\eta)(1-\nu)} < \frac{1-P_{A,1}(1-\nu)}{(1-\nu)(1-P_{A,1})} \nonumber \\
& \frac{\eta + (1-\eta) \nu }{ (1-\eta)} < \frac{1-P_{A,1}(1-\nu)}{1-P_{A,1}} \nonumber \\
&  0 < \frac{1-P_{A,1}(1-\nu)}{(1-P_{A,1})} - \frac{\eta + (1-\eta) \nu }{ (1-\eta)} \label{linear}
\end{align} 
We follow a similar linearity argument to above.
For $\nu=1$, Line \ref{linear} becomes $\frac{1}{1-p_{A,1}} - \frac{1}{1-\eta} > 0$.
This holds if $1-p_{A,1} < 1-\eta \iff \eta < p_{A,1}$.  This is clearly true because the negative region of $h_A^*$ has exactly an $\eta$ fraction of positives, and the positive region of $h_A^*$ has a $1-\eta>\eta$ fraction of positives.
For $\nu=0$, Line \ref{linear} becomes $1-\frac{\eta}{1-\eta} >0 $ which holds since $0< \eta < 1/2$.
\comment{
\begin{align*}
    & \eta < p_{A,1} = p(1-\eta) + (1-p) \eta = p - p\eta + \eta - p\eta \\
    & 0 < p-2p\eta
\end{align*}
The last line holds since $\eta < 1 /2$.
Observe that the left hand constraint
$ \frac{(1-\eta) (1-\nu) }{ \eta + (1-\eta) \nu } < 0 $ since $ Z > 0$. 
However since $\eta  \in (0,1/2)$ and $\nu \in [0,1]$ as long as the 
right hand side is well defined (i.e. $\eta  \neq 0$) and $\nu \neq 1$ then this requirement is satisfied.
Note that $Z$ from Section \ref{labelbias} i.e. $Z = \frac{1-P_{A,+}(1-\nu)}{(1-\nu)(1-P_{A,+})}$ satisfies these constraints-this is consistent with our figures.
\begin{align*}
&  \frac{1-\eta + \eta \nu }{\eta (1-\nu)}  > \frac{\eta + (1-\eta) \nu }{ (1-\eta)(1-\nu)} 
\end{align*}

\begin{align*}
  \frac{1-\eta + \eta \nu }{\eta (1-\nu)} & > \frac{\eta + (1-\eta) \nu }{ (1-\eta)(1-\nu)} \\
((1-\eta) + \eta \nu )(1-\eta) & > \eta (\eta + (1-\eta)\nu ) \\
    1-\eta + \eta \nu - \eta +\eta^{2} - \eta^2 \nu & > \eta^2 + \eta \nu - \eta^{2} \nu \\
 1- 2 \eta + \eta^2 & > \eta^2 \\
 1- 2 \eta & > 0 
\end{align*}
This always holds so we can always find the correct $Z$ satisfying both constraints. 
Thus to find the optimal $Z$ we can simply increment $Z$ from $0$ until the classifier snaps out of 
saying $0$ on all samples from the dis-advantaged group and before it snaps into saying $1$ for all
samples on the advantaged group. 
}
\section{Calibration Results} \label{calib}
\begin{thm} \label{calibinformal}
Assume the training data is corrupted by Under-Representation Bias with parameter $\beta < 1$.
For any such $\beta$, $h^*$ does not satisfy Calibration on the biased data and
thus Calibration constrained ERM will return a hypothesis that has strictly worse true error than the true error of $h^*$.
This occurs even when $(1-\eta) \beta > \eta $, i.e. in the bias regime such that plain ERM on the biased data would recover $h^*$.

Moreover, if bias is such that $(1-\eta)\beta < \eta$ and thus ERM on the biased data will not recover $h^*$, then the unique ERM solution that satisfies Calibration on the biased data is a trivial classifier, meaning that all individuals from Group $A$ receive one label (the positive label) and all individuals from Group $B$ receive the opposite label.
\end{thm}
\begin{proof}
Recall that Calibration of hypothesis $h=(h_A, h_B)$ requires that both Eq. \ref{calp} and \ref{caln} hold simultaneously.
\begin{align}
& P_{x \sim \mathscr{D}_A}(y=1 | h_A (x)=1 ) = P_{x \sim \mathscr{D}_B}( y = 1 | h_B (x)=1)  \label{calp} \\
& P_{x \sim \mathscr{D}_A}(y=1 | h_A (x)=0 ) = P_{x \sim \mathscr{D}_B}( y = 1 | h_B (x)=0) \label{caln}
\end{align}
We assume that if one of the terms is vacuous in the Calibration constraints 
, then that constraint is still satisfied. 
In other words, if one bin is non-empty for one group while the corresponding bin for the other group is empty, we assume that bin satisfies Calibration. 
Due to the effects of the bias model positive samples from Group $B$ appear in the training data with lowered frequency
and so the equalities in Equations \ref{calp} and \ref{caln} become: 
\begin{align}
& P_{x \sim A}(y=1 | h_{A}^{*} (x)=1 ) > P_{x \sim B}( y = 1 | h_{B}^{*} (x)=1)  \label{cal1} \\
& P_{x \sim A}(y=1 | h_{A}^{*} (x)=0 ) > P_{x \sim B}( y = 1 | h_{B}^{*} (x)=0) \label{cal2}
\end{align}
Thus $h^*  = ( h_{A}^*, h_{B}^{*})$ violates calibration for any $\beta < 1$ and any other hypothesis satisfying calibration will have strictly greater 
error on the true data distribution.
Intuitively, for $h$ to be Calibrated it will need to reduce the left-hand side of Equation \ref{cal1} because
it cannot increase the right-hand side and will have to increase the right-hand side of Equation \ref{cal2} because it cannot decrease the left-hand side.
As a result, its true error will be strictly larger than that of $h^*$.

Now, consider $(1-\eta)\beta < \eta$.
In this case, plain ERM will not recover $h^*$.
With this amount of bias, then:
\[ P_{x\sim A} (y = 1 | h_{A}^{*}(x) =1 ) > P_{x\sim A} (y = 1 | h_{A}^{*}(x) =0  ) >  P_{x\sim B} (y = 1 | h_{B}^{*}(x) =1 ) 
> P_{x\sim B} (y = 1 | h_{B}^{*}(x) =10)  \]

Satisfying Calibration with non-trivial classifiers requires achieving an equality with one side being a non-negative combination of the first two probabilities, and the other side being a non-negative combination of the second two probabilities.
Since these inequalities are all strict, this is clearly not possible, so the only way to satisfy calibration is to use a trivial classifier that assigns all of Group $A$ to one label, and all of Group $B$ to the other label.\footnote{Which trivial classifier is selected by ERM will depend on $p$ and $r$. If $1-r > r$ and $p > 1/2$, then Group $A$ will be all positive and Group $B$ all negative. While if $1-r > r $ and $p< 1/2$, then then Group $A$ will be all positive and Group $B$ all negative.}
\end{proof}

\comment{
\begin{thm} \label{calibformal}
Assume the training data is corrupted by data subject to Under-Representation Bias with parameter $\beta$.                                                                                   
If $(1-\eta)\beta > \eta$ and a non-trivial classifier satisfying Calibration exists (vanilla ERM will recover $h^*$), then excess true error 
(excess over true error of $h^*$) of the ERM solution satisfying Calibration will be  
\[ (1-r) (1-p)v_A  + r p[\frac{(1-p)^2}{p^2} v_{A} \frac{(1-\eta)\beta + \eta }{\eta \beta + (1-\eta)} ] \]
where \[ v_A =\frac{p}{1-p}\frac{(1-\eta) \eta (1-\beta)}{ \beta (1-\eta)^2 - \eta^2 } \]

If $(1-\eta)\beta < \eta$, then the only ERM solution satisfying Calibration are the two trivial solutions that label all samples from Group $A$ as positive and label all samples from Group $B$ as negative (or vice-versa).
\end{thm}
Observe that as $\beta(1-\eta) - \eta$ approaches zero from the positive direction, the error of the above Calibration satisfying classifier approaches the excess error of trivial classifier. This is even while plain ERM would recover $h^*$.
\begin{proof}
Define $z_A$ such that $z_A := P_{x \sim \mathscr{D}_A} (h_A (x) =1 |  h_{A}^{*}(x) = 1 )$, e.g. the fraction of the  positive region of $h_{A}^{*}$ that our classifier $h_A$ correctly classifies as positive.
(in the previous notation $z_A$ is $\frac{p-p_{1A}}{p}$.
Alternatively, let $v_A := P_{x \sim \mathscr{D}_B} (h_{A}(x) = 1| h_{A}^{*}(x) = 0 )$. 
Let $z_B, v_B$, be defined similarly.
These parameters determine whether classifiers $h_A, h_B$ satisfy calibration.

Recall that Calibration requires that both Eq. \ref{calp} and \ref{caln} hold simultaneously.
\comment{
\begin{align}
& P_{x \sim \mathscr{D}_A}(y=1 | h_A (x)=1 ) = P_{x \sim \mathscr{D}_B}( y = 1 | h_B (x)=1)  \label{calp} \\
& P_{x \sim \mathscr{D}_A}(y=1 | h_A (x)=0 ) = P_{x \sim \mathscr{D}_B}( y = 1 | h_B (x)=0) \label{caln}
\end{align}
We assume that if one of the terms is vacuous in the calibration constraints (e.g if $h_B(x) =  1 $ for no $x$), then that constraint is still satisfied. In other words, if one bin is non-empty for one group while the corresponding bin for the other group is empty, we assume that bin satisfies calibration. 

However, due to the effects of the bias model,
\begin{align}
& P_{x \sim A}(y=1 | h_{A}^{*} (x)=1 ) > P_{x \sim B}( y = 1 | h_{B}^{*} (x)=1)  \label{cal1} \\
& P_{x \sim A}(y=1 | h_{A}^{*} (x)=0 ) > P_{x \sim B}( y = 1 | h_{B}^{*} (x)=0) \label{cal2}
\end{align} }
Observe the effect of $\beta$ requires that in order to satisfy calibration, we need to decrease left-hand sides of Eq. \ref{cal1} and \ref{cal2}. 
However, we can only decrease the left-hand side of Eq. \ref{cal1} by adding mass to $h_A$ from the negative region of $h_{A}^*$.
Achieving equality of Eq. \ref{cal2} requires adding mass from the positive region of $h_{B}^{*}$ to the region classified as negative by $h_B$.

Equations \ref{calp} and \ref{caln} are respectively equivalent to requiring that
\begin{align}
& \frac{p(1-\eta)z_A + (1-p)\eta v_A }{pz_A + (1-p)v_A}  = \frac{p\frac{(1-\eta)\beta}{(1-\eta)\beta +\eta} z_B + (1-p)\frac{\beta \eta}{\beta \eta + (1-\eta)} v_B}{p  z_B + (1-p)v_B} \label{poscalib}\\
& \frac{p(1-\eta)(1-z_A) + (1-p)\eta(1-v_A) }{p(1-z_A) + (1-p)(1-v_A)} = \frac{p\frac{(1-\eta) \beta}{(1-\eta)\beta + \eta} (1-z_B) + (1-p)\frac{\eta \beta}{\eta \beta + (1-\eta)} (1-v_B)}{p (1-z_B)+ (1-p)  (1-v_B)} \label{negcalib} 
\end{align}

We first consider Eq. \ref{poscalib}. 
Lemma \ref{caliblem} allows us to assume that the optimal biased data ERM classifier $h= (h_A, h_B)$ that satisfies Calibration has parameters $z_A = 1$ and $v_A > 0$ while $z_B > 0$ and $v_B = 0$ 
\begin{lem} \label{caliblem}
Assume $h=(h_A,h_B)$ is an non-trivial ERM solution on the biased data that also satisfies Calibration (Equations \ref{cal1} and \ref{cal2}) on the biased data (Under-Representation Bias with parameter $\beta$ with $(1-\eta)\beta > \eta$ ).
Then, without loss of generality, we may assume that  $z_A =1, v_A > 0, z_B < 1$, and $v_B = 0 $. 
\end{lem}
We defer the proof of this lemma until after the complete proof of Theorem \ref{calibformal}.

Using that $z_A =1, v_A > 0$ and Eq. \ref{poscalib} simplifies to
\[ \frac{p(1-\eta) + (1-p) \eta v_A }{p+(1-p)v_A} = \frac{(1-\eta)\beta }{(1-\eta)\beta + \eta} \]
Rearranging, we obtain that 
\begin{align*}
 p [(1-\eta)^2 \beta + \eta (1- \eta) -(1-\eta)\beta] &= v_A (1-p) [ (1-\eta) \beta - \eta (1-\eta) \beta - \eta^2 ] \\
& v_A = \frac{p}{1-p} [\frac{ \beta(1-\eta)^2 + \eta (1-\eta) - (1-\eta) \beta}{ \beta (1-\eta)^2 - \eta^2 }]\\ 
& v_A =  \frac{p(1-\eta)}{1-p} [\frac{ \beta(1-\eta) + \eta -\beta}{ \beta (1-\eta)^2 - \eta^2 }]\\
& v_A =  \frac{p}{1-p} [\frac{(1-\eta) \eta (1-\beta)}{ \beta (1-\eta)^2 - \eta^2 }]
\end{align*}
Now we consider Eq. \ref{negcalib}.
A similar argument using $z_A =1$, $v_A >0 $,$z_B < 1$ and $v_B = 0 $ allows us to simplify to 
\[ \eta = \frac{p \frac{(1-\eta) \beta}{(1-\eta) \beta + \eta} (1-z_B) + (1-p)\frac{\eta \beta}{\eta \beta + (1-\eta)} }{p(1-z_B) + (1-p)} \]
Rearranging,
\begin{align*}
p(1-z_B) [ \frac{(1-\eta)\beta}{ (1-\eta) \beta + \eta} - \eta ] = (1-p) [ \eta - \frac{\eta \beta}{ \eta \beta + (1-\eta)}] \\
p(1- z_B) [ \frac{ (1-\eta)^2 \beta - \eta^2 }{  (1-\eta) \beta + \eta}] = (1-p) [ \frac{(1-\eta) \eta (1-\beta) }{ \eta \beta + (1-\eta)}]\\
(1-z_B) = \frac{1-p}{p} \frac{ [(1-\eta)\beta + \eta][(1-\eta) \eta (1-\beta)]}{[\eta \beta +(1-\eta)][(1-\eta)^2 \beta - \eta^2]} \\
(1-z_B) = \frac{(1-p)^2}{p^2} v_{A} \frac{(1-\eta)\beta + \eta }{\eta \beta + (1-\eta)} \\
z_B = 1 - \frac{(1-p)^2}{p^2} v_{A} \frac{(1-\eta)\beta + \eta }{\eta \beta + (1-\eta)}
\end{align*}

$1-z_{B}$ is the fraction of the positive region of $h_{B}^{*}$ that is classified incorrectly as negative.
Similarly, $v_A$ is the fraction of the negative region of $h_{B}^{*}$ that is classified incorrectly as positive.
True error of $h=(h_A,h_B)$ in excess of $h^*$ depends on $(1-r)v_A(1-p)$ and $r p (1-z_B)$, so combining these terms gives
the Theorem Statement. 

\textit{Note that if either $z_B$ or $v_A$ is not in $[0,1]$, then the only Calibration satisfying solution is the same as the next case, e.g. a trivial classifier.}

Now we consider the case when $(1-\eta)\beta < \eta$. 
The proof for this case is the same as Theorem \ref{calibinformal}. 
In summary, for Under-Representation Bias, for all values of $\beta$, calibration-constrained ERM is strictly worse in accuracy than plain ERM. 
\end{proof}

Now we show the proof of the lemma that allows us to assume that when $(1-\eta)\beta > \eta$, then the Calibrated classifier on the biased data that minimizes biased error, $h=(h_A, h_B)$, has parameters $z_A =1, v_A > 0, z_B < 1$,and  $v_B = 0$.
\begin{proof}
We do this by assuming the hypothesis parameters satisfy Calibration 
and then modifying the hypothesis parameters in such a way that decreases error while maintaining Calibration until the parameters have the values
 $z_A =1, v_A > 0, z_B < 1, v_B = 0$.
 
First, we consider Group $A$.
Assume $z_A, v_A, z_B$, and $v_B$ satisfy Calibration with constants $c,d > 0$. 
We can assume this since we assumed $h=(h_A, h_B)$ is a non-trivial classifier. 

Then let $\Delta$ be the amount we increase $z_A$ by while satisfying calibration and $g(\Delta)$ be the term we add to $v_A$ to still
satisfy calibration. $g(\Delta)$ may be positive or negative. 
\begin{align*}
& P_{x \sim A}(y=1 | h_A (x)=1 )=\frac{p(1-\eta)(z_A+\Delta) + (1-p)\eta (v_A + g(\Delta)) }{p(z_A+\Delta) + (1-p)(v_A+g(\Delta))}  =c =P_{x \sim B}( y = 1 | h_B (x)=1)   \\
& P_{x \sim A}(y=1 | h_A (x)=0 )= \frac{p(1-\eta)(1-z_A-\Delta) + (1-p)\eta(1-v_A-g(\Delta)) }{p(1-z_A- \Delta) + (1-p)(1-v_A-g(\Delta))} = d = P_{x \sim B}( y = 1 | h_B (x)=0) 
    \end{align*}
Thus, 
\begin{align*}
& p(1-\eta)(z_A+\Delta) + (1-p)\eta (v_A+g(\Delta)) =c(p(z_A+\Delta) + (1-p)(v_A+g(\Delta)))    \\
& p(1-\eta)(1-z_A-\Delta) + (1-p)\eta(1-v_A-g(\Delta))=d(p(1-z_A-\Delta) + (1-p)(1-v_A-g(\Delta))) 
    \end{align*}
Adding these two equations together, we obtain for a constant $C$ ($C$ depends on terms not involving $\Delta$) that
\begin{align*}
C = d+(c-d)[p \Delta + (1-p) g(\Delta)]
\end{align*}
Thus in order to maintain calibration we need $g(\Delta) = -\frac{p}{1-p} \Delta$.

The change in Error in Group $A$ when we increase $z_A$ by $\Delta$ and decrease $v_A$ by $\Delta$ is:
\[ \Delta[p( \eta - (1-\eta) ) + (1-p)\frac{p}{1-p}*(\eta - (1-\eta) )] < 0 \] 
Since $0 < \eta < 1/2$, then this expression is negative, and so this modification using
$\Delta$  decreases the error in Group $A$. 

Assume we increase $z_A$ by $\Delta$ and decrease $v_A$ by $\Delta \frac{p}{1-p}$. 
We continue this modification process until we hit the boundary of the domain of $z_A$ or $v_A$. 

Either $z_A=1$ first and the process stops or $v_A=0$ and the process stops.
Assume we are in the second case. Since $v_A=0$ and $z_A > 0$, then the left hand side Equation \ref{poscalib} is $(1-\eta)$. 
However, if this is true then we cannot satisfy Calibration since there is no way to have a region of Group $B$ input space that achieves this Calibration value of $(1-\eta)$, due to the bias effects, so we must be violating Calibration.

Recall that the starting parameters $z_A, v_A$ satisifed Calibration and we picked $\Delta$ and $g(Delta)$ to maintain Calibration during the
modification process.
Thus we must be in the first case and $z_A=1$ and $v_A > 0$.

Now we consider Group $B$.
Analogously to the argument for Group $A$, in order to maintain Calibration for Group $B$, while holding the parameters for Group $A$ fixed,
we increase $z_B$ by $\Delta$ and decrease $v_B$ by $\Delta \frac{p}{1-p}$. 
Just like before, this will change error in Group $B$ by this amount:
\begin{align*}
    & \Delta [p   (\frac{\eta}{\eta + (1-\eta) \beta} - \frac{(1-\eta)\beta}{ \eta + (1-\eta) \beta } ) 
    + (1-p) \frac{p}{1-p}( \frac{\eta \beta}{(1-\eta) + \eta \beta } - \frac{ (1-\eta) }{(1-\eta) + \eta \beta })]
\end{align*}
By the assumption that $(1-\eta)\beta > \eta$, the expression above is negative so increasing $\Delta$ decreases error in Group $B$.

Where the Group $B$ analysis diverges is in the case analysis. 
If we are in the first case and $z_B =1 $ and $v_B > 0$ then the right hand side of Equation \ref{negcalib} is
$\eta \beta$, but this would violate Calibration since no mixing of positive  and negative input space from Group $A$ could reach this low value. 
Thus we are in the second case and $z_B > 0$ and $v_B = 0$.
\end{proof}

\textcolor{red}{------END NEW! -------- Check this last comment about approx. Concern is the argument about $v_A,z_A$}

}
\section{Conclusion}
 In this paper we have shown that Equal Opportunity constrained ERM will recover from several forms of training data bias, including Under-Representation Bias (where positive and/or negative examples of the disadvantaged group show up in the training data at a lower rate than their true prevalence in the population) and Labeling Bias (where each positive example from the disadvantaged group is mislabeled as negative with probability $\nu\in (0,1)$), in a clean model where the Bayes-Optimal classifiers $h_A^*,h_B^*$ satisfy most fairness constraints on the {\em true} distribution and the errors of $h_A^*,h_B^*$ are uniformly distributed. 
 
 We also show that a simpler reweighting approach succeeds in some but not all of our models.  This approach can be viewed as reweighting the training data to satisfy Demographic Parity, and then running an unconstrained ERM on the reweighted data (which is different from placing a Demographic Parity constraint on ERM on the actual training data, which does {\em not} work in our bias models). 
 
 Troublingly, we observe that enforcing calibration harms the very group we intended to aid 
 and results in substantially lowered accuracy across both groups, 
 even when the bias
 is small enough that normal ERM would work well.
 This points to more general issues with the class of fairness criteria called outcome tests (of which calibration is one variety) as also observed in \citep{infra}.
 
One limitation of our results is that we have used a stylized model for the generation of labels and how the bias enters the data-set.
However, we believe our results provide
useful insight into how fairness interventions can aid in reducing errors caused by bias in training data.

Even in this simple model, we observe separations between the fairness interventions 
and note that even when 
the fair solution is the right hypothesis in terms of both true accuracy and fairness, 
the fairness interventions can be tricked by the bias in the data.

In our bias models, we observe starkly different behavior of Equal Opportunity and Equalized Odds, two closely related fairness notions, when used to constrain ERM.
This sharp separation recommends that we closely align diagnosing a fairness concern with selecting an intervention, rather than looking towards universal solutions. 
In particular, biased data concerns like those we model in this paper, appear to be both prevalent 
\citep{bertrand2004emily} and difficult to recognize. 

The high-level message of this paper is that fairness interventions need not be in competition with accuracy and may improve classification accuracy if training data is unrepresentative or biased; however these results will be connected to the true data distributions and features of the biased data-generation process. 

\textbf{Acknowledgements} \\
We would like to thank Jon Kleinberg and Manish Raghavan for their helpful and insightful comments on this manuscript. 
%


\bibliography{mybib.bib}{}
\bibliographystyle{plainnat}

\newpage


\comment{
\subsection{Verification Re-Weighting Works for the Labeling Bias Model}  \label{reweightarg}

The way we intervene by Reweighting is we multiply the loss term for mis-classifying positive examples in $B$ by a factor $C$ 
and we multiply the loss term for misclassifying negative examples in group $B$ by $D$.
The goal of this reweighting is that the ratio of positive to negative samples in the regions above 
the hyper-plane is greater than $1$ while the ratio is less than $1$ in the region below the hyperplane.

Thus we need to simultaneously satisfy re-weighted probabilities need to satisfy 
\begin{align*}
& \frac{P(y=1|h^*(x)=1)}{P(y=0 | h^*(x) =1 ) } = \frac{C*[(1-\eta )(1-\nu )] }{ D(\eta + (1-\eta)\nu)}  > 1   \\
& \frac{P(y=1 |h^*(x)= 0 )}{P(y= | h^*(x) = 0  ) } = \frac{C*[\eta (1-\nu)]}{D((1-\eta) + \eta \nu)} < 1\\
\end{align*}
We can replace $C/D$ with $Z:= C/D$ and then the two constraints are equivalent to requiring that 

\begin{align} \label{ilfweight}
& \frac{\eta + (1-\eta) \nu }{ (1-\eta)(1-\nu)} < Z < \frac{1-\eta + \eta \nu }{\eta (1-\nu)} 
\end{align} 
To check that we can satisfy both these constraints, first we verify that the right hand side of Equation \ref{ilfweight} is non-negative.
$ \frac{(1-\eta) (1-\nu) }{ \eta + (1-\eta) \nu } < 0 $ since $ Z > 0$. 
However since $\eta  \in (0,1/2)$ and $\nu \in [0,1]$ as long as the 
right hand side is well defined (i.e. $\eta  \neq 0$) and $\nu \neq 1$ then this requirement is satisfied.
Note that $Z$ from above i.e. $Z = \frac{1-P_{A,+}(1-\nu)}{(1-\nu)(1-P_{A,+})}$ satisfies these constraints-this is consistent with our figures.
\begin{align*}
&  \frac{1-\eta + \eta \nu }{\eta (1-\nu)}  > \frac{\eta + (1-\eta) \nu }{ (1-\eta)(1-\nu)} 
\end{align*}

\begin{align*}
  \frac{1-\eta + \eta \nu }{\eta (1-\nu)} & > \frac{\eta + (1-\eta) \nu }{ (1-\eta)(1-\nu)} \\
   ((1-\eta) + \eta \nu )(1-\eta) & > \eta (\eta + (1-\eta)\nu ) \\
    1-\eta + \eta \nu - \eta +\eta^{2} - \eta^2 \nu & > \eta^2 + \eta \nu - \eta^{2} \nu \\
 1- 2 \eta + \eta^2 & > \eta^2 \\
 1- 2 \eta & > 0 
\end{align*}
This always holds so we can always find the correct $Z$ satisfying both constraints. 
Thus to find the optimal $Z$ we can simply increment $Z$ from $0$ until the classifier snaps out of 
saying $0$ on all samples from the dis-advantaged group and before it snaps into saying $1$ for all
samples on the advantaged group. 
}
\comment{
\begin{proof}
To show the result we simply need to show that the overall error expression is strictly monotone decreasing in $\Delta$.

We first show the case when $p_1 \neq 0$. 
\begin{align*}
& error_{\mathcal{D}_{A}}(h_A) = (1-p)\eta + (p-\Delta) \eta + \Delta (1-\eta) = \eta + \Delta (1-2\eta)\\ \\
&  error_{\mathcal{D}_{B}}(h_B) =  (1-p)(\eta \beta_{POS} (1-\nu)) + \Delta ((1-\eta)\beta_{POS}(1-\nu)) + (p-\Delta) (\eta \beta_{NEG} + (1-\eta) \beta_{POS} \nu) \\
&  error_{\mathcal{D}_{B}}(h_B) =  (1-p)(\eta \beta_{POS} (1-\nu)) +p (\eta \beta_{NEG} + (1-\eta) \beta_{POS} \nu) \\
& + \Delta ((1-\eta)\beta_{POS}(1-\nu) - \eta \beta_{NEG} - (1-\eta) \beta_{POS} \nu)
\end{align*}
Let $h=(h_A, h_B)$ where we use the corresponding classifier for samples from Group $A$ and Group $B$.
Thus the overall error is 
\begin{align*}
& error(h_A, h_B) = P(h(x) \neq y)= (1-\hat{r}) \cdot  error_{\mathcal{D}_{B}}(h_B) + \hat{r} \cdot  error_{\mathcal{D}_{B}}(h_B)
\end{align*}

The $\Delta$ term in the sum above has the coefficient 
\begin{align*}
& (1-\hat{r})(1-2\eta) + \hat{r} ( (1-\eta)\beta_{POS} (1-\nu) - (1-\eta)\beta_{POS}\nu - \eta \beta_{NEG})
\end{align*}
Clearly if the coefficient of $\Delta$ is positive then the error is minimized when $\Delta=0$.
That this coefficient is positive is simply our original assumption.
Sending $\Delta$ to zero reduces error the candidate classifier $h^*$, which we already showed satisfies equal opportunity. 

Now we consider the case where $p_{2A}=p_{2B}=\Delta$ and $p_{1A}= p_{2A} = 0$
\begin{align*}
& error_{\mathcal{D}_{A}}(h_A) = \eta + \Delta (1-2\eta) \\
&  error_{\mathcal{D}_{B}}(h_B) =  p( \eta \beta_{NEG}+(1-\eta)\beta_{POS}\nu) + \Delta((1-\eta)\beta_{NEG} + \eta \beta_{POS} \nu) + 
(1-p-\Delta)(\eta \beta_{+} (1-\nu)
\end{align*}

The terms involving $\Delta$ in the overall sum are
\[ (1-\hat{r})(1-2\eta) + \hat{r} ((1-\eta) \beta_{NEG} -(1-2\nu ) \eta \beta_{POS}) \]

Thus we have shown for any classifier satisfying equal opportunity, in two steps we can reduce strictly its error on
the biased distribution and still satisfy Equal Opportunity by transformation to $h^*$.
Thus $h^*$ has the lowest error of any classifier on the biased training data. 
\end{proof}
}

\comment{

\begin{align*}
& \frac{P(y=1 | h^* (x) = 1 ) }{ P(y=0 | h^*(x) = 1 } < 1 \\
&  \frac{ (1-\eta_1)(1-\eta_2) }{\eta_1 + (1-\eta_1)\eta_2} < 1 \\
&   1-\eta_2-\eta_1 + \eta_1 \eta_2 < \eta_1 +\eta_2 - \eta_1 \eta_2  \\
&   1- 2(\eta_1 + \eta_2 - \eta_1 \eta_2 ) < 0 \\
& \frac{1}{2} < \eta_1 + \eta_2 - \eta_1 \eta_2= \eta_2 + \eta_1 ( 1 - \eta_2) 
\end{align*}
This is not a very restrictive condition. For instance if $\eta_2 = 1/4$  
we merely need that $\eta_1 > 1/3$. 
Here we plot this region
\begin{figure}[H]
  \centering
  \includegraphics[width=0.5\linewidth]{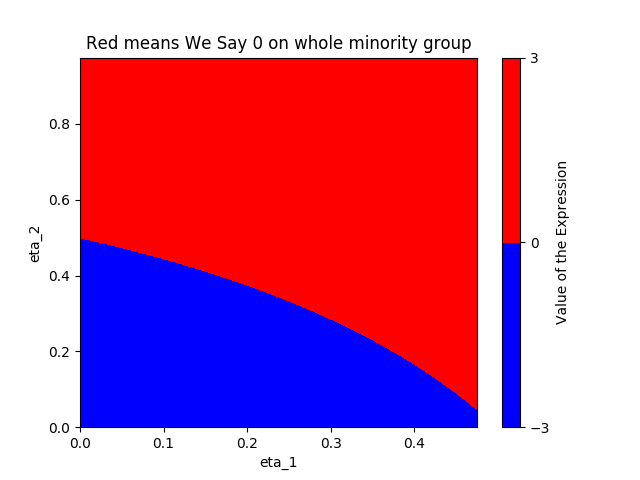}
  \caption{Region where $ILF(\eta_1, \eta_2)$ tricks the ERM classifier with}
\end{figure}

\subsection{Discussion of Region of Validity}
We have proven our result subject to $(1-r)(1-2\eta)+r*((1-\eta)*\beta - \eta) \geq 0$. 
The following plot shows the region where this constraint holds.

\begin{figure}[H]
  \centering
  \includegraphics[width=0.75\linewidth]{Equal Opportunityregion.png}
  \caption{Region where Equal Opportunity \textit{does} recover the Bayes Optimal Classifier. Not red indicates}
\end{figure}

\begin{figure}[H]
\centering
\begin{minipage}{.5\textwidth}
  \centering
  \includegraphics[width=0.75\linewidth]{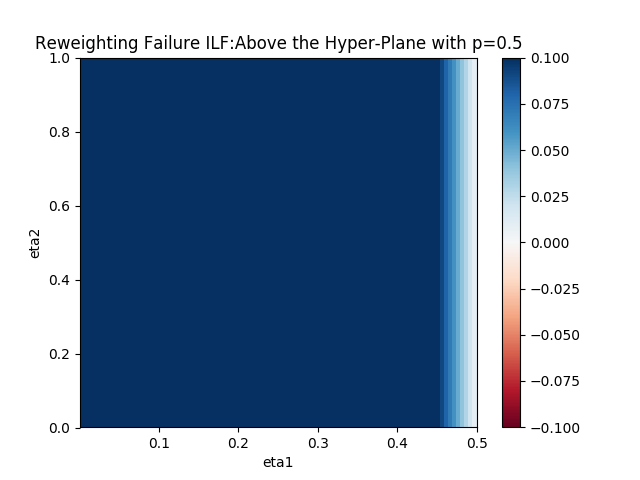}
  \label{fig:test1}
\end{minipage}%
\begin{minipage}{.5\textwidth}
  \centering
  \includegraphics[width=0.75\linewidth]{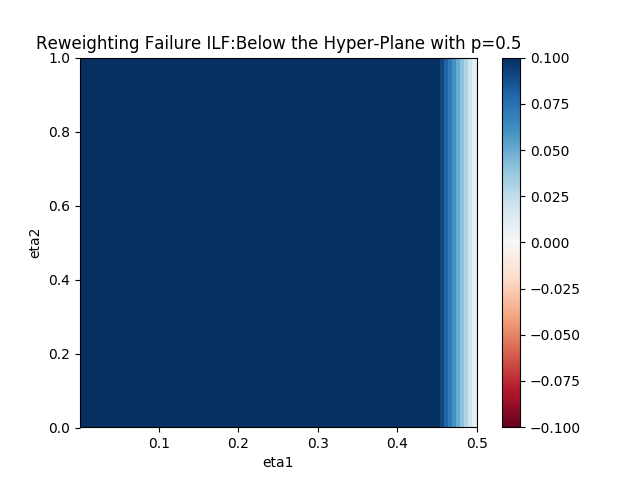}
  \label{fig:test2}
\end{minipage}
\caption{This is a contour plot where the color of the point of the plot indicates the color of the difference of expected positives and expected negatives for the left hand and plot and the difference of 
expected negatveies and expected positives on the right hand plot. If both of these values are positive then the re-weighting intervention recovers the Baye's Optimal classifier. 
As can be seen from these plots nowhere does the sign flip from positive to negative. 
There are regions that approach zero from above but this is becomes a sample complexity issue}.
\end{figure}

\begin{lem}
Note that $h^*$ satisfies equal opportunity on the biased data.
\end{lem}

\begin{proof}
Equal opporunity requires that the true positive rate be the same in both groups.
The true positive of $h_{A}^{*}$ is $P(\hat{Y}=1| a=A, Y=1) = \frac{p(1-\eta_1)}{p(1-\eta_1) + (1-p)\eta_1}$ while the true positive rate for $B$ is 
 $P(\hat{Y}=1| a=A, Y=1) = \frac{p(1-\eta_1)(1-\eta_2) }{p(1-\eta_1)(1-\eta_2) + (1-p)\eta_1(1-\eta_2)}=  P(\hat{Y}=1| a=A, Y=1) = \frac{p(1-\eta_1)}{p(1-\eta_1) + (1-p)\eta_1}$
\end{proof}

\section{ILF  PRIME-PRIME-PROBABLY pointless ignore. }
Another method that may capture the availability heuristic better is that when a negative sample is observed, 
with some probability the sample is duplicated in the training data set.

This is well motivated in that if there is some latent bias in an individual, seeing true negative samples from a dis-advantaged group has a stronger effect and is generalized 
across the entire group. 
This model seems to capture some of the core of predictive policing bias. 
First we try to fix this model with a re-weighting.

\begin{align*}
& \frac{P_A }{ P_{B} } = \frac{X p_{A} }{P_{B}*2*eta_2} \\
& \frac{P_{B}*2*\eta_2}{P_{B} } = X \\
& 2\eta_2 = X 
\end{align*} 

\begin{figure}[H]
\centering
\begin{minipage}{.5\textwidth}
  \centering
  \includegraphics[width=0.5\linewidth]{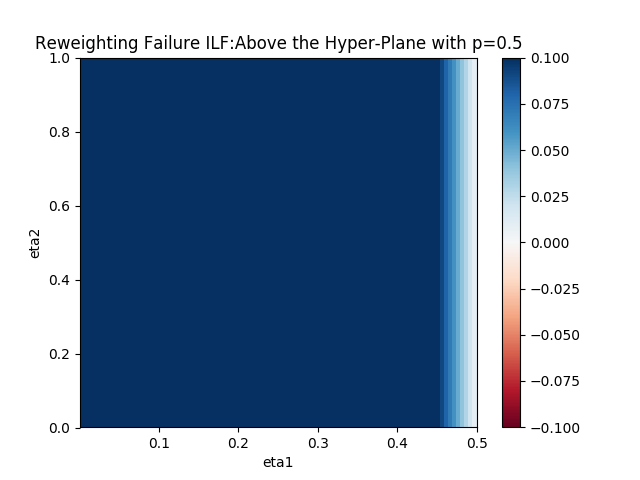}
  \label{fig:test1}
\end{minipage}%
\begin{minipage}{.5\textwidth}
  \centering
  \includegraphics[width=0.5\linewidth]{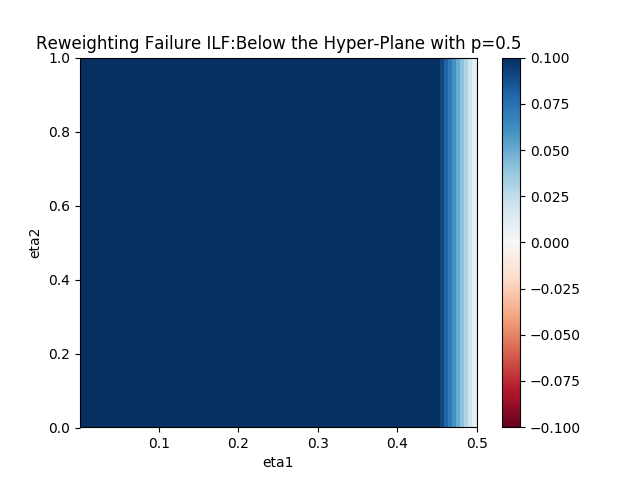}
  \label{fig:test2}
\end{minipage}
\end{figure}

}
\end{document}